\documentclass{article}\usepackage[]{graphicx}\usepackage[]{color}
\makeatletter
\def\maxwidth{ %
  \ifdim\Gin@nat@width>\linewidth
    \linewidth
  \else
    \Gin@nat@width
  \fi
}
\makeatother

\definecolor{fgcolor}{rgb}{0.345, 0.345, 0.345}

\usepackage{framed}
\makeatletter
 {\par\unskip\endMakeFramed%
 \at@end@of@kframe}
\makeatother

\definecolor{shadecolor}{rgb}{.97, .97, .97}
\definecolor{messagecolor}{rgb}{0, 0, 0}
\definecolor{warningcolor}{rgb}{1, 0, 1}
\definecolor{errorcolor}{rgb}{1, 0, 0}
\newenvironment{knitrout}{}{} 

\usepackage{alltt} 

\usepackage{etoolbox}
\newtoggle{arxivformat}
\toggletrue{arxivformat}

\iftoggle{arxivformat} {
  \usepackage[sort&compress, numbers]{natbib}
  \usepackage{times}
}{
  \usepackage{nips15submit_e,times}
  \usepackage[sort&compress, numbers]{natbib}
}

\usepackage{hyperref}

\usepackage[margin=1.5in]{geometry}

\usepackage{graphicx} 
\usepackage{subcaption}
\usepackage{caption}

\usepackage{amsthm}
\usepackage{amsmath}
\usepackage{appendix}

\usepackage{amssymb}
\usepackage{bbm}

\newcommand{\app}[1]{Appendix~\ref{app:#1}}
\newcommand{\lem}[1]{Lemma~\ref{lem:#1}}

\newcommand{\mysec}[1]{Section~\ref{sec:#1}}

\newcommand{\eq}[1]{Eq.~(\ref{eq:#1})}

\newcommand{\eqw}[1]{Eq.~(#1)}
\newcommand{\fig}[1]{Fig.~(\ref{fig:#1})}

\newcommand{\npq}{\eta} 
\newcommand{\mpq}{m} 
\newcommand{\mpopt}{m^*} 
\newcommand{\npopt}{\eta^*} 
\newcommand{\gauss}{\mathcal{N}} 
\newcommand{\truecov}{\Sigma} 
\newcommand{\lrcov}{\hat{\Sigma}} 
\newcommand{\vbcov}{V} 
\newcommand{\constant}{C} 
\newcommand{\klshort}{E}

\theoremstyle{plain}
\newtheorem{theorem}{Theorem}[section]

\newtheorem{proposition}[theorem]{Proposition}
\newtheorem{lemma}[theorem]{Lemma}

\newcommand{\kl}{\textrm{KL}}
\DeclareMathOperator*{\argmin}{arg\,min}
\DeclareMathOperator*{\argmax}{arg\,max}
\newcommand{\mbe}{\mathbb{E}}
\newcommand{\mbeq}{\mathbb{E}_{q}}

\newcommand{\cov}{\textrm{Cov}}
\newcommand{\iid}{\stackrel{iid}{\sim}}
\newcommand{\indep}{\stackrel{indep}{\sim}}

\title{Linear Response Methods for Accurate Covariance Estimates from
       Mean Field Variational Bayes}

\usepackage[scaled=0.7]{beramono}

\author{
Ryan Giordano\\
UC Berkeley\\
\texttt{rgiordano@berkeley.edu}
\and
Tamara Broderick \\
MIT\\
\texttt{tbroderick@csail.mit.edu}
\and
Michael Jordan \\
UC Berkeley\\
\texttt{jordan@cs.berkeley.edu}
}

\iftoggle{arxivformat} {
}{
  \nipsfinalcopy 
}
\IfFileExists{upquote.sty}{\usepackage{upquote}}{}
\begin{document}

\maketitle

\begin{abstract}

Mean field variational Bayes (MFVB) is a popular posterior approximation method
due to its fast runtime on large-scale data sets. However, a well known major
failing of MFVB is that it underestimates the uncertainty of model variables
(sometimes severely) and provides no information about model variable
covariance.
We generalize linear response methods from statistical physics
to deliver accurate uncertainty estimates for model
variables---both for individual variables and coherently across variables.
We call our method \emph{linear response variational Bayes} (LRVB).
When the MFVB posterior approximation is in the exponential family,
LRVB has a simple, analytic form, even for
non-conjugate models. Indeed, we make no assumptions about the form of the
true posterior.
We demonstrate the accuracy and scalability of
our method on a range of models for both simulated and real data.

\end{abstract}


\section{Introduction}\label{sec:intro}

With increasingly efficient data collection methods, scientists are interested
in quickly analyzing ever larger data sets. In particular, the promise of these
large data sets is not simply to fit old models but instead to learn more
nuanced patterns from data than has been possible in the past. In theory, the
Bayesian paradigm yields exactly these desiderata. Hierarchical modeling allows
practitioners to capture complex relationships between variables of interest.
Moreover, Bayesian analysis allows practitioners to quantify the uncertainty in
any model estimates---and to do so coherently across all of the model variables.

\emph{Mean field variational Bayes} (MFVB), a method for approximating
a Bayesian posterior distribution, has grown in
popularity due to its fast runtime on large-scale data sets
\citep{blei:2003:lda, blei:2006:dp, hoffman:2013:stochastic}.
But a well known major failing of MFVB is that it gives
underestimates of the uncertainty of model variables that can be arbitrarily
bad, even when approximating a simple multivariate Gaussian distribution
\citep{mackay:2003:information,bishop:2006:pattern,turner:2011:two}.
Also, MFVB provides no information about how
the uncertainties in different model variables interact
\citep{wang:2005:inadequacy, bishop:2006:pattern, rue:2009:approximate, turner:2011:two}.

By generalizing linear response methods from statistical physics
\citep{parisi:1988:statistical, opper:2003:variational, opper:2001:advancedmeanfield, tanaka:2000:information}
to exponential family variational posteriors, we develop a methodology that
augments MFVB to deliver accurate uncertainty estimates for model
variables---both for individual variables and coherently across variables. In
particular, as we elaborate in \mysec{lr}, when the approximating posterior in
MFVB is in the exponential family, MFVB defines a fixed-point equation in the
means of the approximating posterior, and our approach yields a covariance
estimate by perturbing this fixed point. We call our method \emph{linear
response variational Bayes} (LRVB).

We provide a simple, intuitive formula for calculating the linear response
correction by solving a linear system based on the MFVB solution
(\mysec{lr_subsection}). We show how the sparsity of this system for many common
statistical models may be exploited for scalable computation
(\mysec{scaling_formulas}). We demonstrate the wide applicability of LRVB by
working through a diverse set of models to show that the LRVB covariance
estimates are nearly identical to those produced by a Markov Chain Monte Carlo
(MCMC) sampler, even when MFVB variance is dramatically underestimated
(\mysec{experiments}). Finally, we focus in more depth on models for finite
mixtures of multivariate Gaussians (\mysec{normal_mixture_model}), which have
historically been a sticking point for MFVB covariance estimates
\citep{bishop:2006:pattern,turner:2011:two}. We show that LRVB can give accurate
covariance estimates orders of magnitude faster than MCMC
(\mysec{normal_mixture_model}). We demonstrate both theoretically and
empirically that, for this Gaussian mixture model, LRVB scales linearly in the
number of data points and approximately cubically in the dimension of the
parameter space (\mysec{gmm_scaling}).

\paragraph{Previous Work.}

Linear response methods originated in the statistical physics literature
\citep{opper:2001:advancedmeanfield, tanaka:2000:information,
kappen:1998:efficient, opper:2003:variational}. These methods have been applied
to find new learning algorithms for Boltzmann machines
\citep{kappen:1998:efficient}, covariance estimates for discrete factor graphs
\citep{welling:2004:linear}, and independent component analysis
\citep{hojen:2002:mean}. \citep{tanaka:1998:mean} states that linear response
methods could be applied to general exponential family models but works out
details only for Boltzmann machines. \citep{opper:2003:variational}, which is
closest in spirit to the present work, derives general linear response
corrections to variational approximations; indeed, the authors go further to
formulate linear response as the first term in a functional Taylor expansion to
calculate full pairwise joint marginals. However, it may not be obvious to the
practitioner how to apply the general formulas of \citep{opper:2003:variational}.
Our contributions in the present work are (1) the provision of concrete,
straightforward formulas for covariance correction that are fast and easy to
compute, (2) demonstrations of the success of our method on a wide range of new
models, and (3) an
\href{https://github.com/rgiordan/LinearResponseVariationalBayesNIPS2015}{accompanying suite of code}.

\section{Linear response covariance estimation} \label{sec:lr}

\subsection{Variational Inference}

Suppose we observe $N$ data points, denoted by the $N$-long column vector $x$,
and denote our unobserved model parameters by $\theta$. Here, $\theta$ is a
column vector residing in some space $\Theta$; it has $J$ subgroups and total
dimension $D$. Our model is specified by a distribution of the observed data
given the model parameters---the likelihood $p(x | \theta)$---and a prior
distributional belief on the model parameters $p(\theta)$. Bayes' Theorem yields
the posterior $p(\theta | x)$.

Mean-field variational Bayes (MFVB) approximates $p(\theta | x)$ by a factorized
distribution of the form $q(\theta) = \prod_{j=1}^{J} q(\theta_j)$. $q$ is
chosen so that the Kullback-Liebler divergence $\kl(q || p)$ between $q$ and $p$
is minimized. Equivalently, $q$ is chosen so that $\klshort := L + S$, for $L :=
\mbe_q[\log p (\theta | x)]$ (the expected log posterior) and $S := -
\mbe_q[\log q(\theta)]$ (the entropy of the variational distribution),
is maximized:
\begin{align} \label{eq:kl}
  q^{*} &:= \argmin_{q} \kl(q || p) = \argmin_{q} \mbe_{q} \left[ \log q(\theta) - \log p(\theta | x)  \right] = \argmax_{q} E.
\end{align}
Up to a constant in $\theta$, the objective $\klshort$ is sometimes called the
``evidence lower bound'', or the ELBO \citep{bishop:2006:pattern}. In what
follows, we further assume that our variational distribution,
$q\left(\theta\right)$, is in the exponential family with natural parameter
$\npq$ and log partition function $A$:
$
\log q\left(\theta \vert \npq \right) =
  \npq^{T}\theta - A\left(\npq\right)
$
(expressed with respect to some base measure in $\theta$). We assume that
$p\left(\theta \vert x\right)$ is expressed with respect to the same base
measure in $\theta$ as for $q$. Below, we will make only mild regularity
assumptions about the true posterior $p(\theta | x)$ and no assumptions about
its form.

If we assume additionally that the parameters $\npopt$ at the optimum
$q^*(\theta) = q(\theta | \npopt)$ are in the interior of the feasible space,
then $q(\theta | \npq)$ may instead be described by the mean parameterization:
$\mpq := \mbe_{q} \theta$ with $\mpopt := \mbe_{q^*} \theta$. Thus, the
objective $\klshort$ can be expressed as a function of $m$, and the first-order
condition for the optimality of $q^*$ becomes the fixed point equation
\begin{equation}
  \label{eq:fixed_pt}
  \left. \frac{\partial \klshort}{\partial \mpq} \right|_{\mpq = \mpopt} = 0
  \;
  \Leftrightarrow
  \;
  \left. \left( \frac{\partial \klshort}{\partial \mpq} + \mpq \right) \right|_{\mpq = \mpopt} = \mpopt
  \;
  \Leftrightarrow
  \;
  M(\mpopt) = \mpopt
  \textrm{ for } M(\mpq) := \frac{\partial \klshort}{\partial \mpq} + \mpq.
\end{equation}

\subsection{Linear Response}\label{sec:lr_subsection}

Let $\vbcov$ denote the covariance matrix of $\theta$ under
the variational distribution $q^{*}(\theta)$, and let $\truecov$ denote the
covariance matrix of $\theta$ under the true posterior,
$p(\theta | x)$:
$$
\vbcov := \cov_{q^{*}} \theta,
\quad \quad
\truecov := \cov_{p} \theta.
$$
In MFVB, $\vbcov$ may be a poor estimator of $\truecov$, even when $\mpopt
\approx \mbe_{p} \theta$, i.e., when the marginal estimated means match well
\citep{wang:2005:inadequacy, bishop:2006:pattern, turner:2011:two}. Our goal is
to use the MFVB solution and linear response methods to construct an improved
estimator for $\truecov$.  We will focus on the covariance of the natural
sufficient statistic $\theta$, though the covariance of functions of $\theta$
can be estimated similarly (see \app{function_covariance}).

The essential idea of linear response is to perturb the first-order condition
$M(\mpopt) = \mpopt$ around its optimum. In particular, define the distribution
$p_{t}\left(\theta\vert x\right)$ as a log-linear perturbation of the posterior:
\begin{eqnarray} \label{eq:perturbed_dens}
\log p_{t}\left(\theta\vert x \right) & := &
    \log p\left(\theta\vert x \right) + t^{T}\theta - \constant\left( t\right),
\end{eqnarray}
where $\constant\left( t\right)$ is a constant in $\theta$. We assume that $p_t
(\theta \vert x)$ is a well-defined distribution for any $t$ in an open ball
around 0. Since $\constant\left( t\right)$ normalizes $p_t(\theta \vert x)$, it
is in fact the cumulant-generating function of $p(\theta \vert x)$, so the
derivatives of $\constant\left( t\right)$ evaluated at $t=0$ give the cumulants
of $\theta$. To see why this perturbation may be useful, recall that the second
cumulant of a distribution is the covariance matrix, our desired estimand:
$$
  \truecov = \cov_{p}(\theta) = \left. \frac{d}{dt^T dt} C(t) \right|_{t=0} = \left. \frac{d}{dt^T} \mbe_{p_t} \theta \right|_{t=0}.
$$

The practical success of MFVB relies on the fact that its estimates of the mean
are often good in practice. So we assume that $\mpopt_t \approx \mbe_{p_t}
\theta$, where $\mpopt_t$ is the mean parameter characterizing $q_t^*$ and
$q_t^*$ is the MFVB approximation to $p_t$. (We examine this assumption further
in \mysec{experiments}.) Taking derivatives with respect to $t$ on both sides of
this mean approximation and setting $t=0$ yields
\begin{equation}\label{eq:lrvb_derivative_defn}
  \truecov = \cov_{p}(\theta) \approx \left. \frac{d\mpq^*_t}{dt^T} \right|_{t=0} =: \lrcov,
\end{equation}
where we call $\lrcov$ the \emph{linear response variational Bayes} (LRVB)
estimate of the posterior covariance of $\theta$.

We next show that there exists a simple formula for $\lrcov$.
Recalling the form of the KL divergence (see~\eq{kl}), we have that
$-\kl(q || p_t) = E + t^{T} m =: E_t$. Then by \eq{fixed_pt}, we have
$\mpopt_t = M_t(\mpopt_t)$ for $M_t(\mpq) := M(\mpq) + t$. It follows from
the chain rule that
\begin{equation}\label{eq:dM_dt}
  \frac{d\mpq^*_t}{dt}
    = \left. \frac{\partial M_t}{\partial \mpq^T} \right|_{\mpq = \mpopt_t}
      \frac{d\mpopt_t}{dt} + \frac{\partial M_t}{\partial t}
    = \left. \frac{\partial M_t}{\partial \mpq^T} \right|_{\mpq = \mpopt_t}
      \frac{d\mpopt_t}{dt} + I,
\end{equation}
where $I$ is the identity matrix.  If we assume that we are at a strict local
optimum and so can invert the Hessian of $E$, then evaluating at $t=0$ yields
\begin{equation}
  \label{eq:gen_lrcov}
  \lrcov = \left. \frac{d\mpq^*_t}{dt^T} \right|_{t=0} = \frac{\partial M}{\partial \mpq} \lrcov + I
    = \left(\frac{\partial^2 \klshort}{\partial \mpq \partial \mpq^T} + I \right) \lrcov + I
    \quad
    \Rightarrow
    \quad
    \lrcov = -\left(\frac{\partial^2 \klshort}{\partial \mpq \partial \mpq^T} \right)^{-1},
\end{equation}
where we have used the form for $M$ in \eq{fixed_pt}. So the LRVB estimator
$\lrcov$ is the negative inverse Hessian of the optimization objective, $E$, as
a function of the mean parameters. It follows from \eq{gen_lrcov} that $\lrcov$
is both symmetric and positive definite when the variational distribution
$q^{*}$ is at least a local maximum of $E$.

We can further simplify \eq{gen_lrcov} by using the exponential family form of
the variational approximating distribution $q$. For $q$ in exponential family
form as above, the negative entropy $-S$ is dual to the log partition function
$A$ \citep{wainwright2008graphical}, so $S = -\npq^T \mpq + A(\npq)$; hence,
$$
  \frac{dS}{dm}
    = \frac{\partial S}{\partial \npq^T} \frac{d\npq}{d\mpq} + \frac{\partial S}{\partial \mpq}
    = \left(\frac{\partial A}{\partial \npq} - \mpq \right) \frac{d\npq}{d\mpq} - \npq(\mpq)
    = - \npq(\mpq).
$$
Recall that for exponential families, $\partial \npq(\mpq) / \partial \mpq =
V^{-1}$. So \eq{gen_lrcov} becomes\footnote{For a comparison of this formula
with the frequentist ``supplemented expectation-maximization'' procedure see
\app{SEM}.}
\begin{align}
  \nonumber
  \lrcov = -\left(\frac{\partial^2 L}{\partial m \partial m^T} + \frac{\partial^2 S}{\partial m \partial m^T}\right)^{-1}
    &= -(H - V^{-1})^{-1}, \textrm{ for } H := \frac{\partial^2 L}{\partial m \partial m^T}.  \Rightarrow\\
  \label{eq:spec_lrvb}
  \lrcov &= (I - VH)^{-1} V.
\end{align}

When the true posterior $p(\theta | x)$ is in the exponential family and
contains no products of the variational moment parameters, then $H=0$ and
$\lrcov=\vbcov$. In this case, the mean field assumption is correct, and the
LRVB and MFVB covariances coincide at the true posterior covariance.
Furthermore, even when the variational assumptions fail, as long as certain mean
parameters are estimated exactly, then this formula is also exact for
covariances. E.g., notably, MFVB is well-known to provide arbitrarily bad
estimates of the covariance of a multivariate normal posterior
\citep{mackay:2003:information,wang:2005:inadequacy,bishop:2006:pattern,turner:2011:two},
but since MFVB estimates the means exactly, LRVB estimates the covariance exactly
(see~\app{mvn_exact}).

\subsection{Scaling the matrix inverse} \label{sec:scaling_formulas}

\eq{spec_lrvb} requires the inverse of a matrix as large
as the parameter dimension of the posterior $p(\theta | x)$,
which may be computationally prohibitive.
Suppose we are interested in the covariance of parameter sub-vector $\alpha$,
and let $z$ denote the remaining parameters: $\theta = \left( \alpha, z \right)^{T}$.
We can partition
$
\truecov = \left( \truecov_{\alpha}, \truecov_{\alpha z}; \truecov_{z\alpha}, \truecov_{z} \right).
$
Similar partitions exist for $\vbcov$ and $H$.
If we assume a mean-field factorization $q(\alpha,z) = q(\alpha)q(z)$, then
$\vbcov_{\alpha z} = 0$. (The variational
distributions may factor further as well.)
We calculate the Schur complement of $\lrcov$ in \eq{spec_lrvb}
with respect to its $z$th
component to find that 
\begin{equation} \label{eq:nuisance_lrvb_est}
\hat{\truecov}_{\alpha} =
 ( I_{\alpha} - V_{\alpha}H_{\alpha} -
  V_{\alpha}H_{\alpha z} \left(I_{z} - V_{z}H_{z})^{-1}
  V_{z}H_{z\alpha}\right)^{-1} V_{\alpha}.
\end{equation}
Here, $I_\alpha$ and $I_z$ refer to $\alpha$- and $z$-sized identity
matrices, respectively.  In cases where
$\left(I_{z} - V_{z}H_{z}\right)^{-1}$
can be efficiently calculated (e.g., all the experiments
in \mysec{experiments}; see \fig{sparsity_patterns} in \app{np_details}),
\eq{nuisance_lrvb_est}
requires only an $\alpha$-sized inverse.

\section{Experiments} \label{sec:experiments}

We compare the covariance estimates from LRVB and MFVB in a range of models,
including models both with and without conjugacy
\footnote{All the code is available on our Github repository,
\href{https://github.com/rgiordan/LinearResponseVariationalBayesNIPS2015}{\texttt{rgiordan/LinearResponseVariationalBayesNIPS2015}},
}.
We demonstrate the superiority
of the LRVB estimate to MFVB in all models before focusing in on Gaussian
mixture models for a more detailed scalability analysis.

For each model, we simulate datasets with a range of parameters.  In the graphs,
each point represents the outcome from a single simulation.  The horizontal axis
is always the result from an MCMC procedure, which we take as the ground truth.
As discussed in \mysec{lr_subsection}, the accuracy of the LRVB covariance for a
sufficient statistic depends on the approximation $\mpopt_t \approx \mbe_{p_t}
\theta$. In the models to follow, we focus on regimes of moderate dependence
where this is a reasonable assumption for most of the parameters (see
\mysec{re_simulation} for an exception). Except where explicitly mentioned, the
MFVB means of the parameters of interest coincided well with the MCMC means, so
our key assumption in the LRVB derivations of \mysec{lr} appears to hold.

\subsection{Normal-Poisson model} \label{sec:normal_poisson_model}

\paragraph{Model.}
First consider a Poisson generalized linear mixed model, exhibiting non-conjugacy.
We observe Poisson draws $y_n$ and a design vector $x_n$, for $n=1,...,N$.
Implicitly below, we will everywhere condition on the $x_n$, which
we consider to be a fixed design matrix.
The generative model is:
\begin{align}
  z_n \vert \beta, \tau \indep \gauss\left(z_n \vert \beta x_n, \tau^{-1}\right), &
  \quad
  y_n \vert z_n \indep \textrm{Poisson}\left(y_n \vert \exp(z_n)\right), \label{eq:pn_model}\\
  \beta \sim \gauss( \beta \vert 0, \sigma^2_\beta), &
  \quad
  \tau \sim \Gamma( \tau \vert \alpha_\tau, \beta_\tau).
  \nonumber
\end{align}
For MFVB, we factorize $q\left(\beta,\tau,z\right) =
q\left(\beta\right)q\left(\tau\right)\prod_{n=1}^{N}q\left(z_{n}\right)$.
Inspection reveals that the optimal $q\left(\beta\right)$ will be Gaussian, and
the optimal $q\left(\tau\right)$ will be gamma (see \app{np_details}). Since the
optimal $q\left(z_n\right)$ does not take a standard exponential family form, we
restrict further to Gaussian $q\left(z_{n}\right)$. There are product terms in
$L$ (for example, the term
$\mbeq\left[\tau\right]\mbeq\left[\beta\right]\mbeq\left[z_{n}\right]$), so
$H\ne0$, and the mean field approximation does not hold; we expect LRVB to
improve on the MFVB covariance estimate.  A detailed description of how to
calculate the LRVB estimate can be found in \app{np_details}.

\paragraph{Results.}

\newcommand{\pnn}{500}
\newcommand{\pnnsims}{100}
\newcommand{\pnmupriorvar}{10}
\newcommand{\pntaualpha}{1}
\newcommand{\pntaubeta}{1}
\newcommand{\pnmcmciters}{20000}

We simulated $\pnnsims$ datasets, each with $\pnn$ data points and a randomly
chosen value for $\mu$ and $\tau$.  We drew the design matrix $x$ from a normal
distribution and held it fixed throughout.  We set prior hyperparameters
$\sigma_\beta^{2} = \pnmupriorvar$, $\alpha_\tau = \pntaualpha$, and $\beta_\tau =
\pntaubeta$. To get the ``ground truth'' covariance matrix, we took
$\pnmcmciters$ draws from the posterior with the R \texttt{MCMCglmm} package
\citep{rpackage:MCMCglmm}, which used a combination of Gibbs and Metropolis
Hastings sampling.  Our LRVB estimates used the autodifferentiation software
\texttt{JuMP} \citep{JuMP:LubinDunningIJOC}.

Results are shown in \fig{PN_SimulationResults}. Since $\tau$ is high in many of
the simulations, $z$ and $\beta$ are correlated, and MFVB underestimates the
standard deviation of $\beta$ and $\tau$. LRVB matches the MCMC standard
deviation for all $\beta$, and matches for $\tau$ in all but the most correlated
simulations. When $\tau$ gets very high, the MFVB assumption starts to bias the
point estimates of $\tau$, and the LRVB standard deviations start to differ from
MCMC.  Even in that case, however, the LRVB standard deviations are much more
accurate than the MFVB estimates, which underestimate the uncertainty
dramatically.  The final plot shows that LRVB estimates the covariances of $z$
with $\beta$, $\tau$, and $\log \tau$ reasonably well, while MFVB considers them
independent.

\begin{knitrout}
\definecolor{shadecolor}{rgb}{0.969, 0.969, 0.969}\color{fgcolor}\begin{figure}[ht!]

{\centering \includegraphics[width=0.17\linewidth,height=0.17\linewidth]{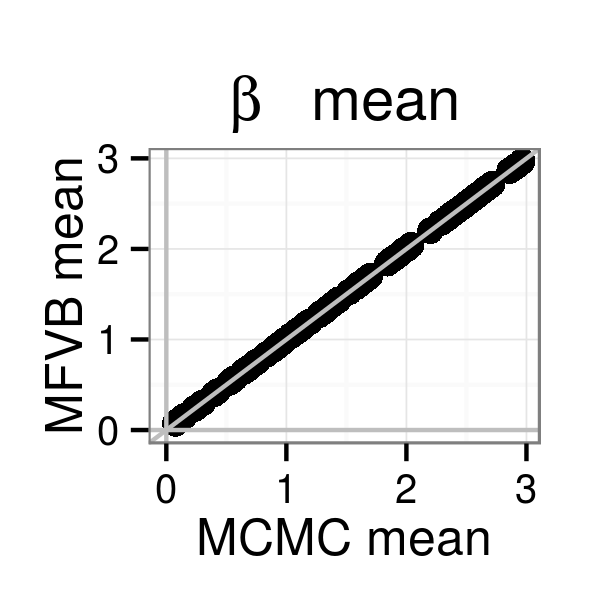}
\includegraphics[width=0.17\linewidth,height=0.17\linewidth]{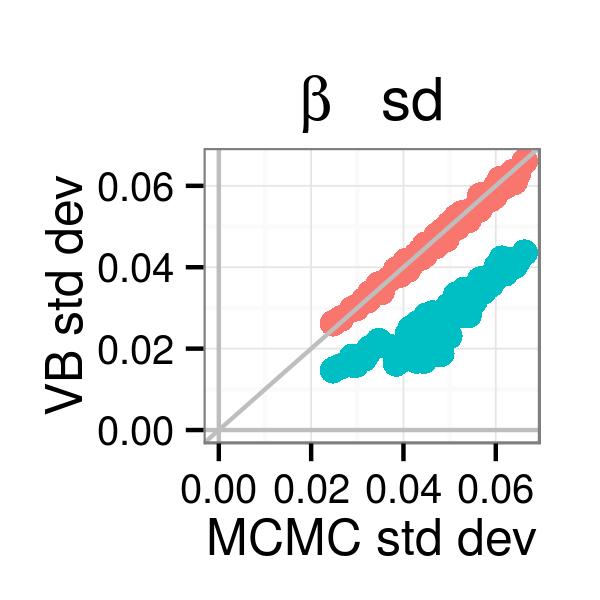}
\includegraphics[width=0.17\linewidth,height=0.17\linewidth]{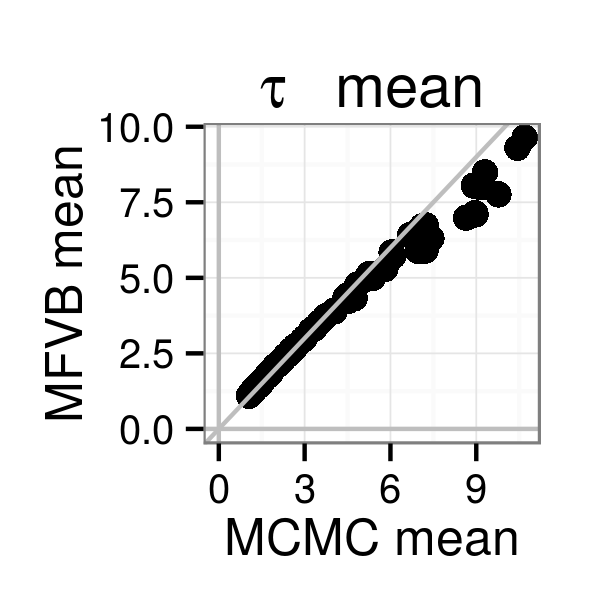}
\includegraphics[width=0.17\linewidth,height=0.17\linewidth]{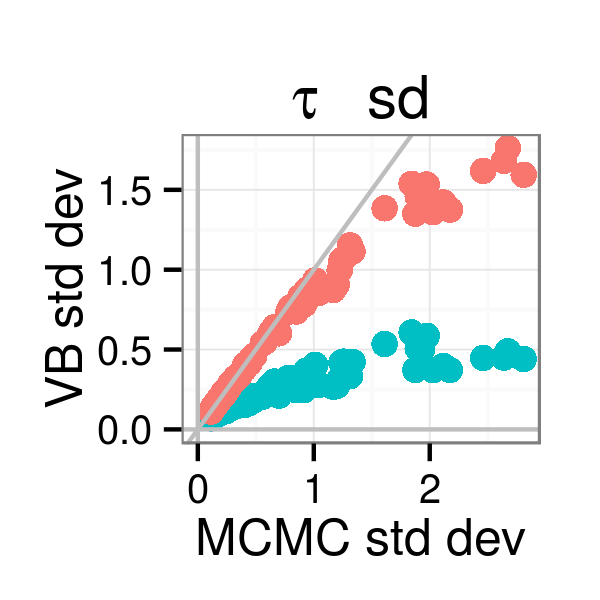}
\includegraphics[width=0.17\linewidth,height=0.17\linewidth]{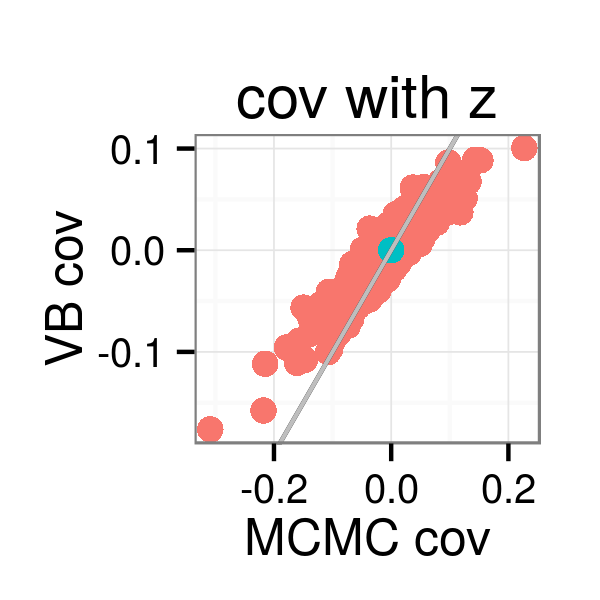}
\includegraphics[width=0.05\linewidth,height=0.17\linewidth]{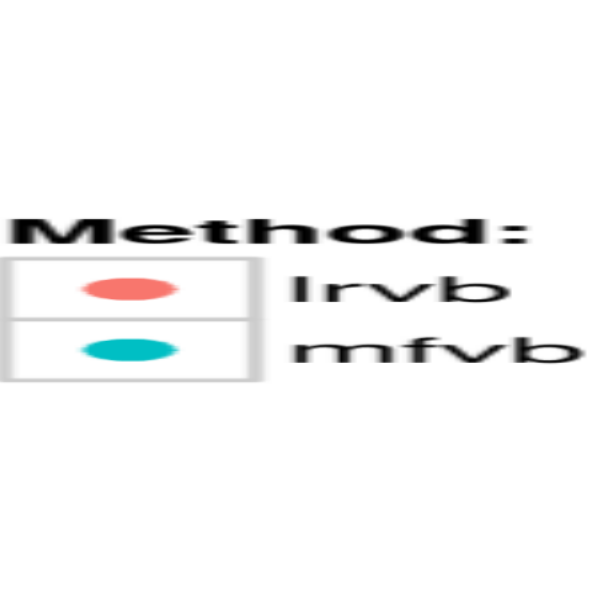}

}

\caption[Posterior mean and covariance estimates on normal-Poisson simulation data]{Posterior mean and covariance estimates on normal-Poisson simulation data.}\label{fig:PN_SimulationResults}
\end{figure}

\end{knitrout}

\subsection{Linear random effects} \label{sec:random_effects_model}

\newcommand{\ren}{300}
\newcommand{\rensims}{100}
\newcommand{\rerenum}{30}
\newcommand{\reNuPriorAlpha}{2}
\newcommand{\reNuPriorGamma}{2}
\newcommand{\reTauPriorAlpha}{2}
\newcommand{\reTauPriorGamma}{2}
\newcommand{\reBetaPriorInfo}{0.1}
\newcommand{\reZSD}{0.4}
\newcommand{\remcmciters}{}

\paragraph{Model.}

Next, we consider a simple random slope linear model, with full details in
\app{re_details}. We observe scalars $y_n$ and $r_n$ and a vector $x_n$, for
$n=1,...,N$. Implicitly below, we will everywhere condition on all the $x_n$ and
$r_n$, which we consider to be fixed design matrices. In general, each random
effect may appear in multiple observations, and the index $k(n)$ indicates which
random effect, $z_k$, affects which observation, $y_n$.  The full generative
model is:
\begin{align*} y_n \vert \beta, z, \tau \indep \gauss\left(y_n \vert
\beta^T x_n + r_n z_{k(n)}, \tau^{-1}\right), &\quad z_k \vert \nu \iid
\gauss\left(z_k \vert 0, \nu^{-1}\right), \\ \beta \sim \gauss(\beta \vert 0,
\Sigma_\beta), \quad \nu \sim \Gamma(\nu \vert \alpha_\nu, \beta_\nu), &\quad
\tau \sim \Gamma(\tau \vert \alpha_\tau, \beta_\tau). \end{align*}
We assume the mean-field factorization $q\left(\beta,\nu,\tau,z\right) =
q\left(\beta\right) q\left(\tau\right) q\left(\nu\right) \prod_{k=1}^{K}
q\left(z_n\right)$.  Since this is a conjugate model, the optimal $q$ will be in
the exponential family with no additional assumptions.

\paragraph{Results.}\label{sec:re_simulation}

We simulated $\rensims$ datasets of $\ren$ datapoints each and $\rerenum$
distinct random effects.  We set prior hyperparameters to $\alpha_\nu =
\reNuPriorAlpha$, $\beta_\nu = \reNuPriorGamma$, $\alpha_\tau =
\reTauPriorAlpha$ , $\beta_\tau = \reTauPriorGamma$, and $\Sigma_\beta =
\reBetaPriorInfo ^ {-1} I$.  Our $x_n$ was 2-dimensional.
As in \mysec{normal_poisson_model},
we implemented the variational solution using the autodifferentiation
software \texttt{JuMP} \citep{JuMP:LubinDunningIJOC}. The MCMC fit was
performed with $\remcmciters$ using \texttt{MCMCglmm} \citep{rpackage:MCMCglmm}.

Intuitively, when the random effect explanatory variables $r_n$ are highly
correlated with the fixed effects $x_n$, then the posteriors for $z$ and $\beta$
will also be correlated, leading to a violation of the mean field assumption and
an underestimated MFVB covariance.  In our simulation, we used $r_n = x_{1n} +
\gauss(0, \reZSD)$, so that $r_n$ is correlated with $x_{1n}$ but not $x_{2n}$.
The result, as seen in \fig{RE_SimulationResults}, is that $\beta_1$ is
underestimated by MFVB, but $\beta_2$ is not. The $\nu$ parameter, in contrast,
is not well-estimated by the MFVB approximation in many of the simulations.
Since the LRVB depends on the approximation $\mpopt_t \approx \mbe_{p_t}
\theta$, its LRVB covariance is not accurate either
(\fig{RE_SimulationResults}). However, LRVB still improves on the MFVB standard
deviation.
\begin{knitrout}
\definecolor{shadecolor}{rgb}{0.969, 0.969, 0.969}\color{fgcolor}\begin{figure}[ht!]

{\centering \includegraphics[width=0.19\linewidth,height=0.19\linewidth]{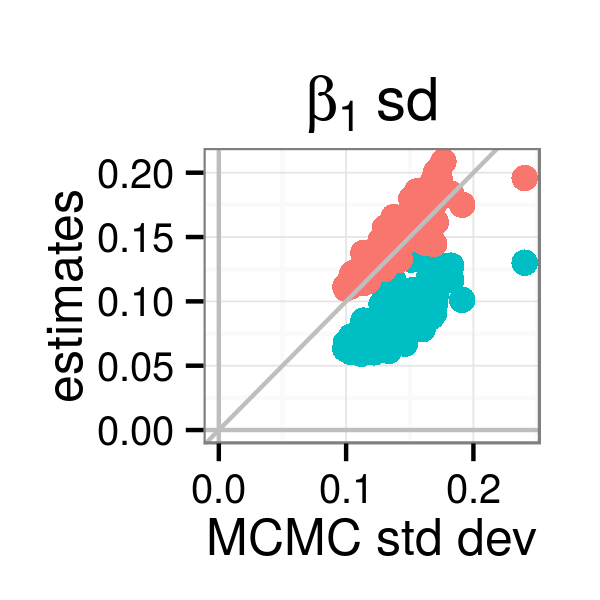}
\includegraphics[width=0.19\linewidth,height=0.19\linewidth]{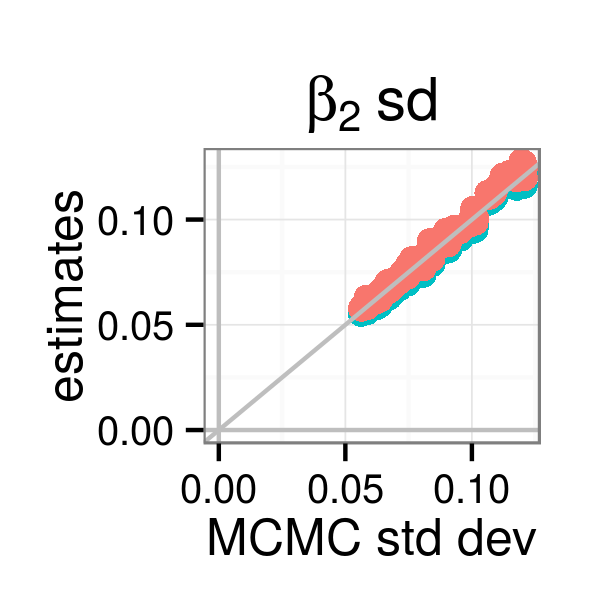}
\includegraphics[width=0.19\linewidth,height=0.19\linewidth]{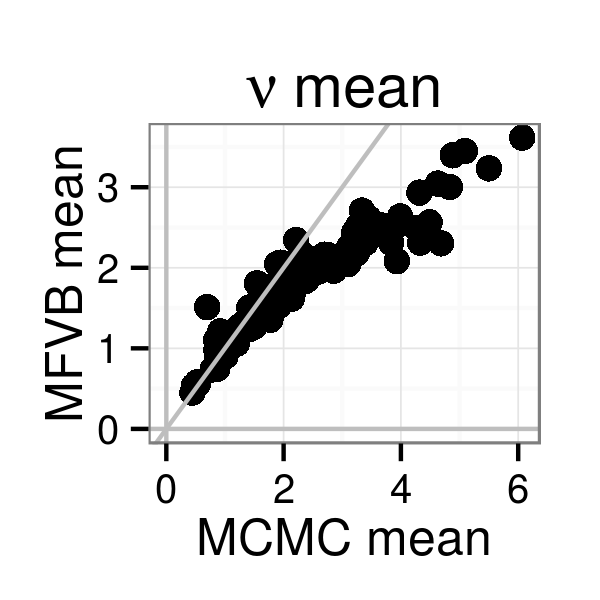}
\includegraphics[width=0.19\linewidth,height=0.19\linewidth]{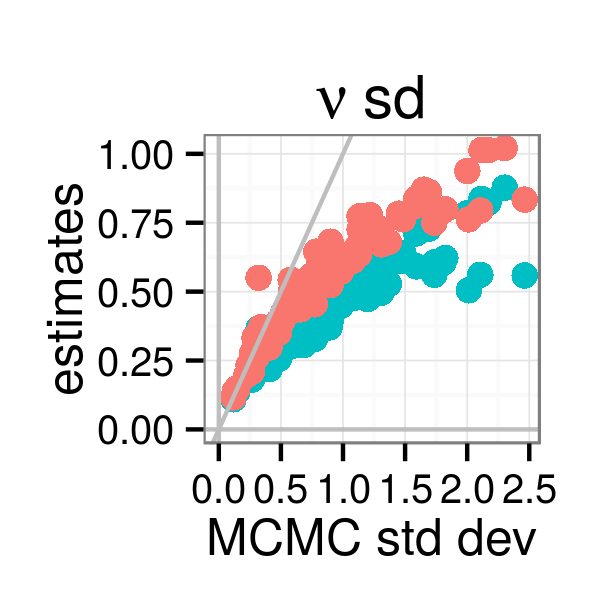}
\includegraphics[width=0.19\linewidth,height=0.19\linewidth]{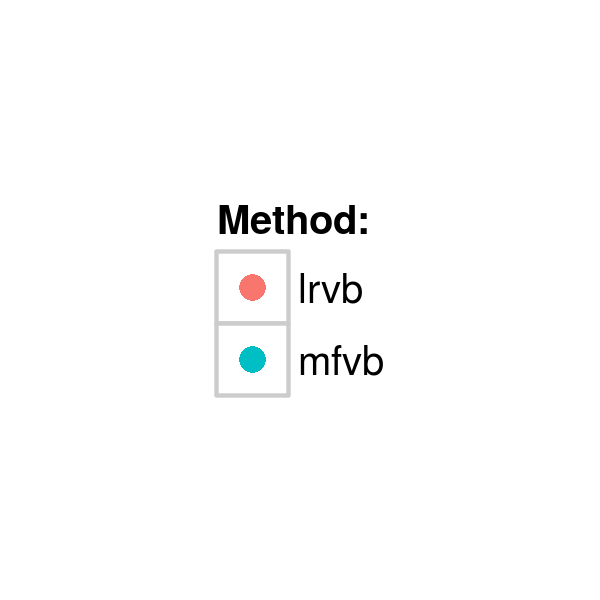}

}

\caption[Posterior mean and covariance estimates on linear random effects simulation data]{Posterior mean and covariance estimates on linear random effects simulation data.}\label{fig:RE_SimulationResults}
\end{figure}

\end{knitrout}
\subsection{Mixture of normals} \label{sec:normal_mixture_model}
\paragraph{Model.}
Mixture models constitute some of the most popular models for MFVB application
\citep{blei:2003:lda, blei:2006:dp} and are often used as an example of where
MFVB covariance estimates may go awry \citep{bishop:2006:pattern,
turner:2011:two}. Thus, we will consider in detail a Gaussian mixture model
(GMM) consisting of a $K$-component mixture of $P$-dimensional multivariate
normals with unknown component means, covariances, and weights. In what follows,
the weight $\pi_k$ is the probability of the $k$th component, $\mu_k$ is the
$P$-dimensional mean of the $k$th component, and $\Lambda_k$ is the $P \times P$
precision matrix of the $k$th component (so $\Lambda_k^{-1}$ is the covariance
parameter).  $N$ is the number of data points, and $x_{n}$ is the $n$th observed
$P$-dimensional data point. We employ the standard trick of augmenting the data
generating process with the latent indicator variables $z_{nk}$, for $n=1,...,N$
and $k=1,...,K$, such that $z_{nk} = 1$ implies $x_{n} \sim \gauss(\mu_k,
\Lambda^{-1}_k)$. So the generative model is:
\begin{align}
P(z_{nk} = 1) = \pi_k, & \quad
p(x | \pi, \mu, \Lambda, z) =
    \prod_{n=1:N} \prod_{k=1:K} \gauss(x_n | \mu_k, \Lambda^{-1}_k)^{z_{nk}}
    \label{eq:normal_mixture_model}
\end{align}
We used diffuse conditionally conjugate priors (see \app{mvn_details} for
details). We make the variational assumption $q\left(\mu, \pi, \Lambda, z\right) =
\prod_{k=1}^K
q\left(\mu_k\right)q\left(\Lambda_k\right)q\left(\pi_k\right)\prod_{n=1}^N
q\left(z_{n}\right)$. We compare the accuracy and speed of our estimates to
Gibbs sampling on the augmented model (\eq{normal_mixture_model}) using the
function \texttt{rnmixGibbs} from the R package \texttt{bayesm}.  We implemented
LRVB in C++, making extensive use of
\texttt{RcppEigen}~\citep{rpackage:RcppEigen}. We evaluate our results both on
simulated data and on the MNIST data set~\citep{mnist:lecun1998gradient}.
\newcommand{\MNISTn}{12665}
\newcommand{\MNISTTestN}{2115}
\newcommand{\MNISTp}{25}
\newcommand{\MNISTTestAccuracy}{0.92}
\newcommand{\MNISTTestError}{0.08}

\newcommand{\GMMeffsizecutoff}{500}
\newcommand{\GMMsimulationsize}{198}
\newcommand{\GMMsimulationn}{10000}
\newcommand{\GMMsimulationp}{2}
\newcommand{\GMMsimulationk}{2}
\newcommand{\GMMsimulationvbtime}{2.82}
\newcommand{\GMMsimulationgibbstime}{444.44}

\paragraph{Results.} 

For simulations, we generated $N=\GMMsimulationn$ data points from
$K=\GMMsimulationk$ multivariate normal components in
$P=\GMMsimulationp$ dimensions.  MFVB is expected
to underestimate the marginal variance of $\mu$, $\Lambda$, and $\log(\pi)$
when the components overlap since that induces correlation in the
posteriors due to the uncertain classification of points between the
clusters. We check the covariances estimated with
\eq{spec_lrvb} against a Gibbs sampler, which we treat as the ground
truth.\footnote{The likelihood described in \mysec{normal_mixture_model} is symmetric under
relabeling.  When the component locations and shapes have
a real-life interpretation, the researcher is generally
interested in the uncertainty of $\mu$, $\Lambda$, and $\pi$ for a
particular labeling, not the
marginal uncertainty over all possible re-labelings.  This poses
a problem for standard MCMC methods, and we restrict our simulations
to regimes where label switching did not occur in our Gibbs sampler.
The MFVB solution conveniently avoids this problem since the mean field
assumption prevents it from representing more than one mode of the
joint posterior.}

We performed $\GMMsimulationsize$ simulations, each of which had
at least $\GMMeffsizecutoff$ effective Gibbs
samples in each variable---calculated with the R tool \texttt{effectiveSize}
from the \texttt{coda} package \citep{rpackage:coda}.
The first three plots show the diagonal standard deviations,
and the third plot shows the off-diagonal covariances.  Note
that the off-diagonal covariance plot excludes the MFVB estimates since most
of the values are zero.
\fig{SimulationStandardDeviations} shows that the
raw MFVB covariance estimates are often quite different from the
Gibbs sampler results, while the LRVB estimates match
the Gibbs sampler closely.

For a real-world example, we fit a $K=2$ GMM
to the $N=\MNISTn$ instances of handwritten $0$s and $1$s in the
MNIST data set. We used PCA to reduce the pixel intensities to $P=\MNISTp$
dimensions. Full details are provided in \app{mnist_details}.
In this MNIST analysis, the $\Lambda$
standard deviations were under-estimated by MFVB
but correctly estimated by LRVB (\fig{SimulationStandardDeviations}); the other parameter standard deviations
were estimated correctly by both and are not shown.
\begin{knitrout}
\definecolor{shadecolor}{rgb}{0.969, 0.969, 0.969}\color{fgcolor}\begin{figure}[ht!]

{\centering \includegraphics[width=0.19\linewidth,height=0.19\linewidth]{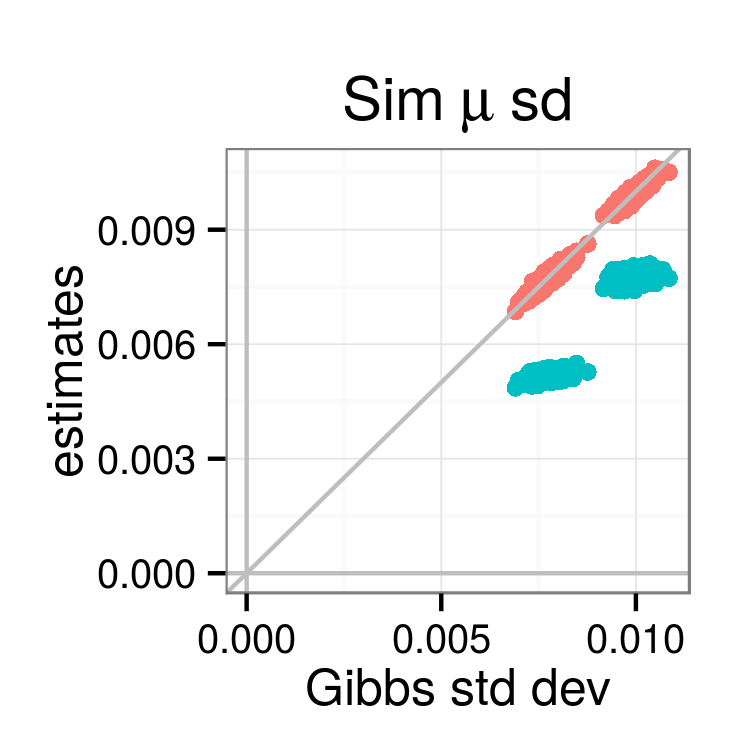}
\includegraphics[width=0.19\linewidth,height=0.19\linewidth]{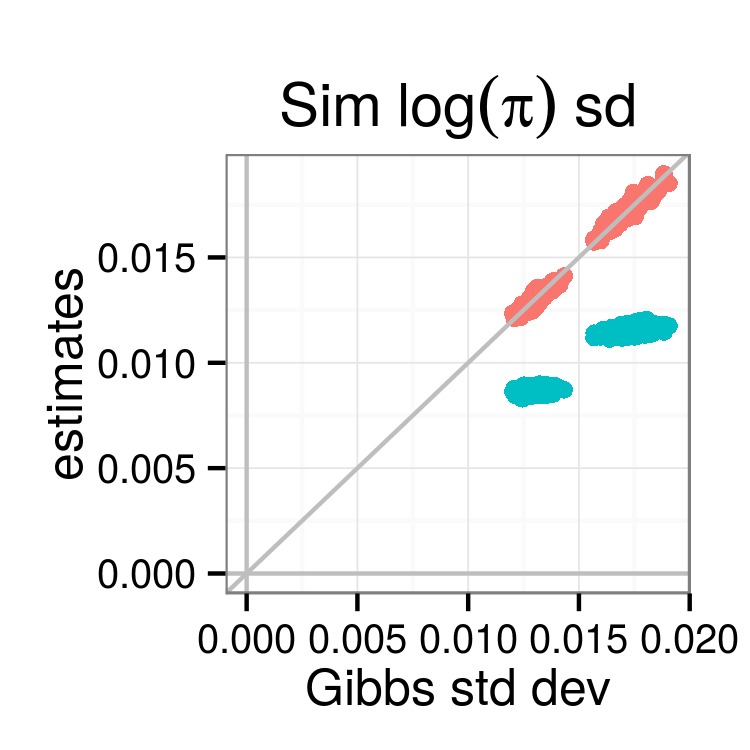}
\includegraphics[width=0.19\linewidth,height=0.19\linewidth]{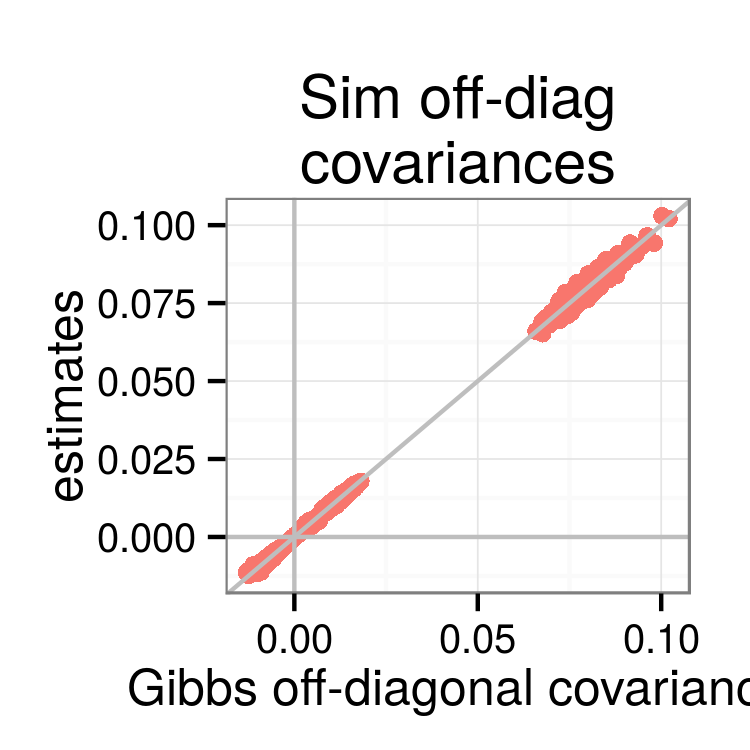}
\includegraphics[width=0.19\linewidth,height=0.19\linewidth]{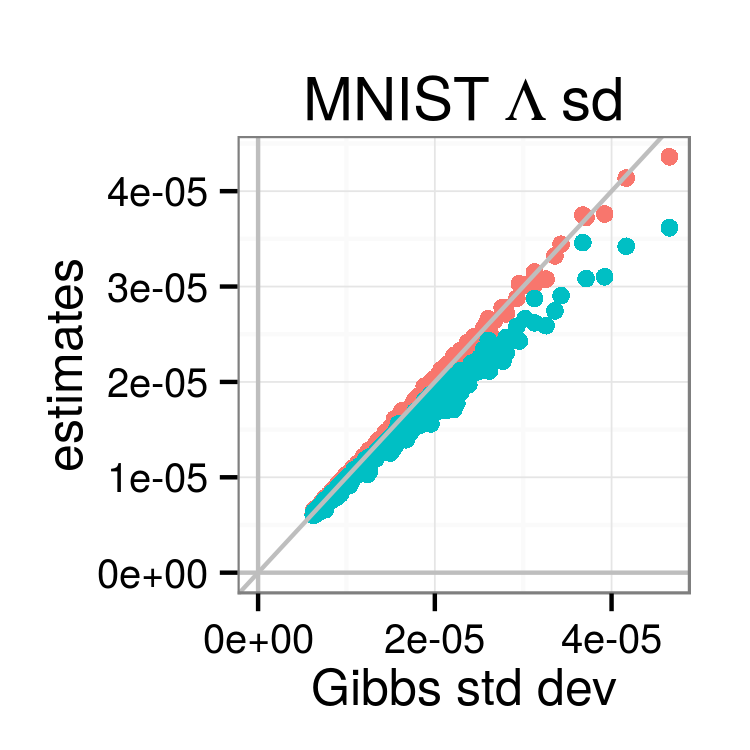}
\includegraphics[width=0.19\linewidth,height=0.19\linewidth]{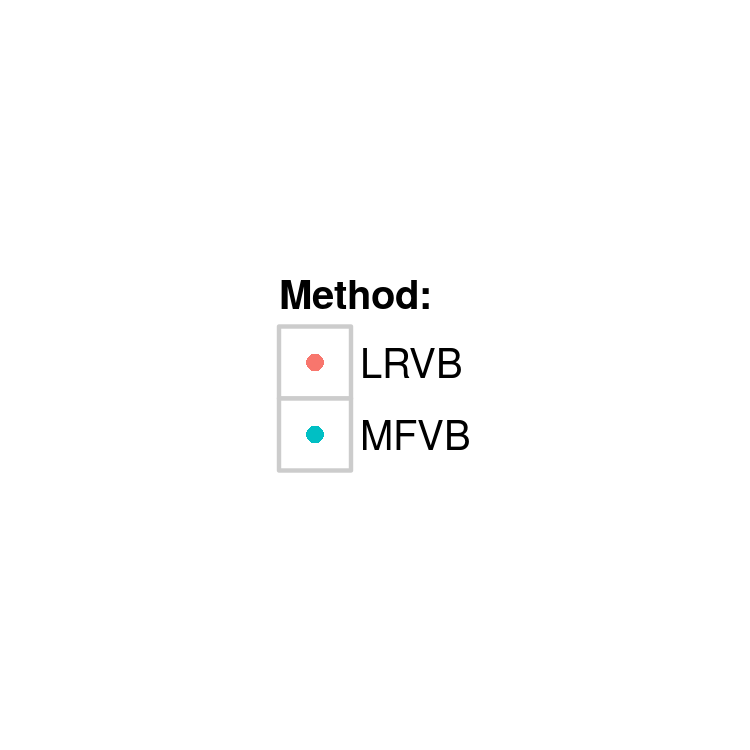}

}

\caption[Posterior mean and covariance estimates on GMM simulation and MNIST data]{Posterior mean and covariance estimates on GMM simulation and MNIST data.}\label{fig:SimulationStandardDeviations}
\end{figure}

\end{knitrout}
\subsection{Scaling experiments} \label{sec:gmm_scaling}

We here explore the computational scaling of LRVB in more depth
for the finite Gaussian mixture model (\mysec{normal_mixture_model}).
In the terms of \mysec{scaling_formulas}, $\alpha$
includes the sufficient statistics from $\mu$, $\pi$, and $\Lambda$,
and grows as $O(KP^2)$.
The sufficient statistics for the variational posterior of
$\mu$ contain the $P$-length vectors $\mu_k$, for each
$k$, and the $(P + 1) P / 2$ second-order products
in the covariance matrix $\mu_k \mu_k^T$.  Similarly, for each $k$,
the variational posterior of $\Lambda$ involves the
$(P + 1) P / 2$ sufficient statistics in the symmetric matrix
$\Lambda_k$ as well as the term $\log |\Lambda_k|$.  The
sufficient statistics for the posterior of $\pi_k$ are the $K$
terms $\log \pi_k$.\footnote{Since $\sum_{k=1}^{K} \pi_k = 1$, using $K$
sufficient statistics involves one redundant parameter.
However, this does not violate any of the necessary assumptions
for \eq{spec_lrvb}, and it considerably simplifies the calculations.
Note that though the perturbation argument of \mysec{lr}
requires the parameters of
$p(\theta | x)$ to be in the interior of the feasible space,
it does not require that the parameters of $p(x | \theta)$
be interior.}
So, minimally, \eq{spec_lrvb}
will require the inverse of a matrix of size $O(KP^2)$.
The sufficient statistics for
$z$ have dimension $K \times N$.  Though
the number of parameters thus grows with the number of
data points, $H_{z}=0$ for the multivariate normal
(see \app{mvn_details}),
so we can apply \eq{nuisance_lrvb_est} to replace the
inverse of an $O(KN)$-sized matrix with multiplication by the same matrix.
Since a matrix inverse is cubic in the size of the matrix,
the worst-case scaling for LRVB is then $O(K^2)$ in $K$,
$O(P^6)$ in $P$, and $O(N)$ in $N$.

In our simulations (\fig{ScalingGraphs}) we can see that,
in practice, LRVB scales linearly\footnote{The Gibbs sampling time was linearly rescaled to the amount
of time necessary to achieve 1000 effective samples in the slowest-mixing
component of any parameter.  Interestingly, this rescaling
leads to increasing efficiency in the Gibbs sampling at low $P$ due
to improved mixing, though the benefits cease to accrue at moderate dimensions.}
 in $N$
and approximately cubically in $P$ across the dimensions considered.\footnote{For numeric stability
we started the optimization procedures for MFVB at the true
values, so the time to compute the optimum in our simulations
was very fast and not representative of practice.
On real data, the optimization time will depend on the
quality of the starting point.
Consequently, the times shown for LRVB are only the
times to compute the LRVB estimate.  The optimization times were
on the same order.}
The $P$ scaling is presumably better than the theoretical worst
case of $O(P^6)$ due to extra efficiency in the numerical linear algebra.
Note that the vertical axis of the leftmost plot is on the log scale.
At all the values of $N$, $K$ and $P$
considered here, LRVB was at least as fast as Gibbs sampling and
often orders of magnitude faster.
\begin{knitrout}
\definecolor{shadecolor}{rgb}{0.969, 0.969, 0.969}\color{fgcolor}\begin{figure}[ht!]

{\centering \includegraphics[width=0.3\linewidth,height=0.25\linewidth]{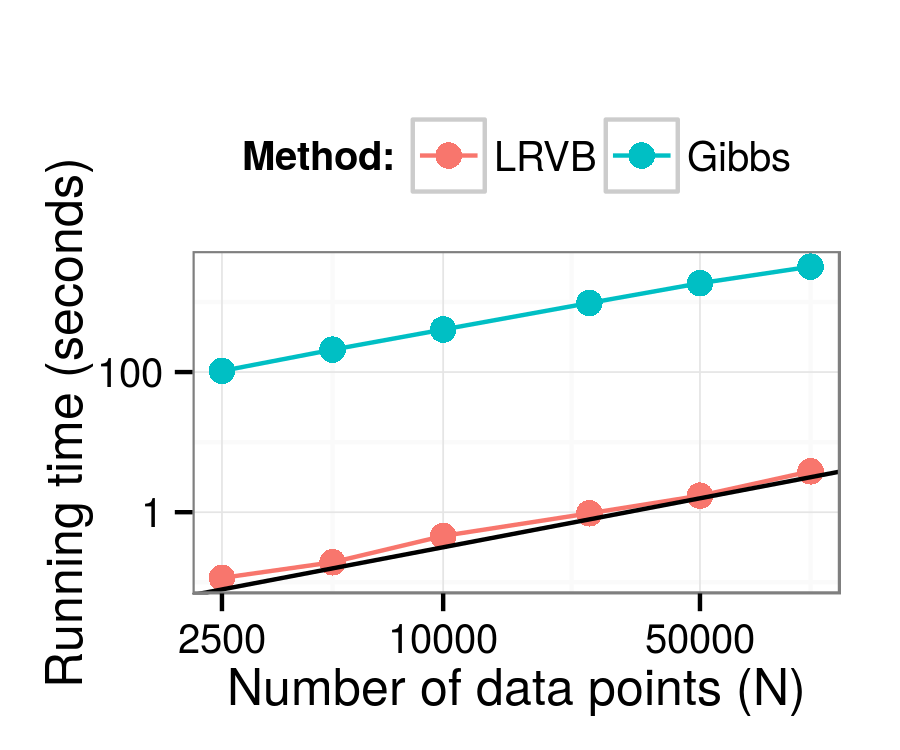}
\includegraphics[width=0.3\linewidth,height=0.25\linewidth]{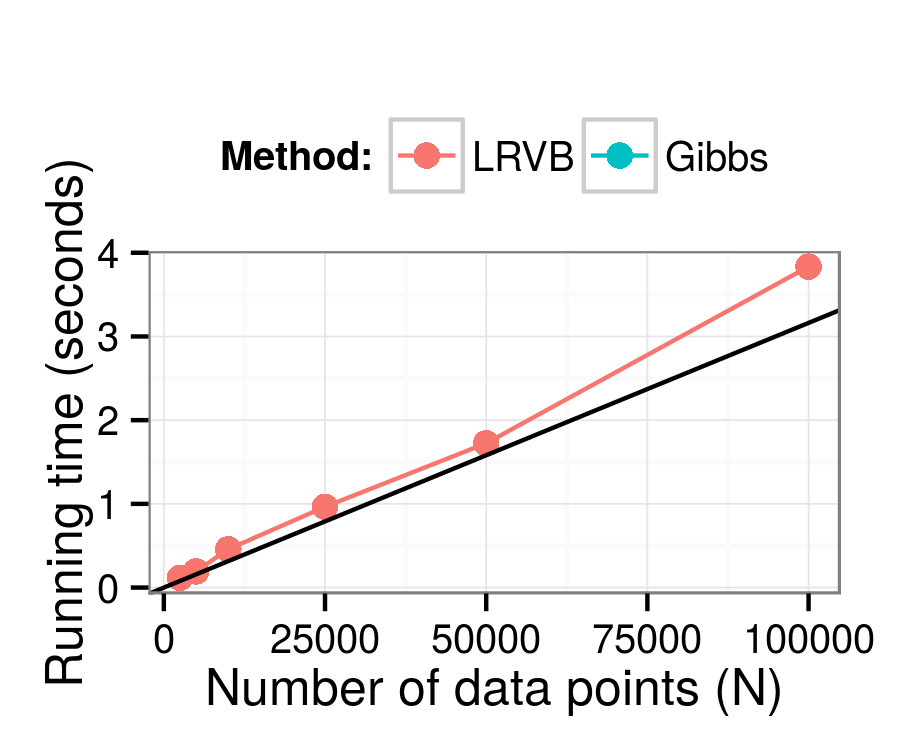}
\includegraphics[width=0.3\linewidth,height=0.25\linewidth]{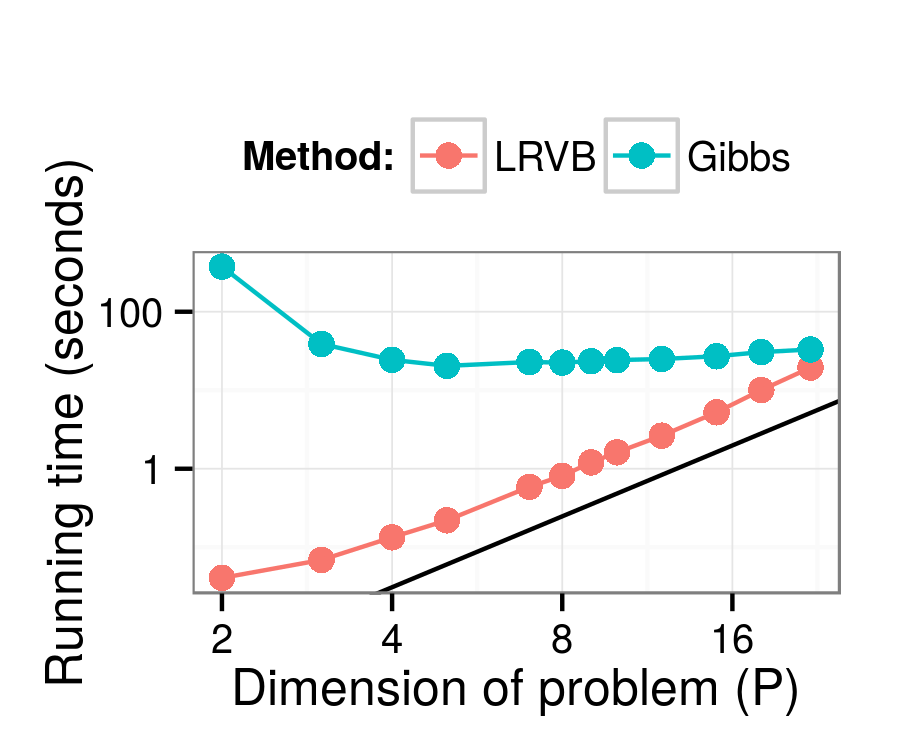}

}

\caption[Scaling of LRVB and Gibbs on simulation data in both log and linear scales]{Scaling of LRVB and Gibbs on simulation data in both log and linear scales.  Before taking logs, the line in the two lefthand (N) graphs is $y\propto x$, and in the righthand (P) graph, it is $y \propto x^3$.}\label{fig:ScalingGraphs}
\end{figure}

\end{knitrout}

\section{Conclusion} \label{sec:conclusion}

The lack of accurate covariance estimates from the widely used mean-field
variational Bayes (MFVB) methodology has been a longstanding shortcoming of
MFVB. We have demonstrated that in sparse models, our method, linear response
variational Bayes (LRVB), can correct MFVB to deliver these covariance estimates
in time that scales linearly with the number of data points.  Furthermore, we
provide an easy-to-use formula for applying LRVB to a wide range of inference
problems.
Our experiments on a diverse set of models have demonstrated the efficacy of
LRVB, and our detailed study of scaling of mixtures of multivariate Gaussians
shows that LRVB can be considerably faster than traditional MCMC methods. We
hope that in future work our results can be extended to more complex models,
including Bayesian nonparametric models, where MFVB has proven its practical
success.

\paragraph{Acknowledgments.}

The authors thank Alex Blocker for helpful comments.
R.~Giordano and T.~Broderick were funded by Berkeley Fellowships.

{\small
\bibliographystyle{plain} 
\bibliography{lrvb_nips_arxiv}
}


\onecolumn
\appendix

\iftoggle{arxivformat} {
}{
  \renewcommand{\appendixpagename}{Supplementary Material}
}

\appendixpage

You can find this paper, as well as all the code necessary to run the
described experiments, in our Github repo,
\href{https://github.com/rgiordan/LinearResponseVariationalBayesNIPS2015}
{\texttt{rgiordan/LinearResponseVariationalBayesNIPS2015}}.

\section{LRVB estimates of the covariance of functions}
\label{app:function_covariance}

In \mysec{lr_subsection}, we derived an estimate of the covariance of the
natural sufficient statistics, $\theta$, of our variational approximation,
$q(\theta)$. In this section we derive a version of \eq{spec_lrvb} for the
covariance of functions of $\theta$.

We begin by estimating the covariance between $\theta$ and a function
$\phi(\theta)$. Suppose we have an MFVB solution, $q(\theta)$, to \eq{kl}.
Define the expectation of $\phi(\theta)$ to be $\mbeq\left[\phi(\theta)\right] :=
f(m)$.  This expectation is function of $m$ alone since $m$ completely
parameterizes $q$. As in \eq{perturbed_dens}, we can consider a perturbed log
likelihood that also includes $f\left(m\right)$:
\begin{eqnarray*}
  \log p_{t}\left(\theta\vert x \right) & = &
    \log p+t_{0}^{T}m+t_{f}f\left(m\right):=\log p+t^{T}m_{f}\\
  t & := & \left(\begin{array}{c}
  t_{0}\\
  t_{f}
  \end{array}\right) \quad \quad
  m_{f} := \left(\begin{array}{c}
  m\\
  f\left(m\right)
  \end{array}\right)
\end{eqnarray*}
Using the same reasoning that led to \eq{lrvb_derivative_defn}, we will define
$$
\truecov_{\theta \phi} = \cov_p(\theta, \phi(\theta))
  \approx \frac{dm_t^{*}}{dt_f} =: \lrcov_{\theta\phi}
$$
We then have the following lemma:
\begin{lemma} \label{lem:theta_function_covariance}
  If $\mbeq\left[\phi(\theta)\right] =: f(m)$ is a differentiable function of $m$
  with gradient $\nabla f$, then
  $$
  \lrcov_{\theta\phi} = \lrcov \nabla f
  $$
\end{lemma}

\begin{proof}
  The derivative of the perturbed ELBO, $E_t$, is given by:
  \begin{eqnarray*}
  E_{t} & := & E+t^{T}m_{f}\\
  \frac{\partial E_{t}}{\partial m} & = &
  \frac{\partial E}{\partial m}+\left(\begin{array}{cc}
  I & \nabla f\end{array}\right)\left(\begin{array}{c}
  t_{0}\\
  t_{f}
  \end{array}\right)\\
  \end{eqnarray*}
  The fixed point \eq{fixed_pt} then gives:
  \begin{eqnarray*}
  M_{t}\left(m\right) & := & M\left(m\right)+\left(\begin{array}{cc}
  I & \nabla f\end{array}\right)\left(\begin{array}{c}
  t_{0}\\
  t_{f}
  \end{array}\right)\\
  \frac{dm_{t}^{*}}{dt^{T}} & = &
    \left.\frac{\partial M_{t}}{\partial m^{T}}\right|_{_{m=m_{t}^{*}}}
    \frac{dm_{t}^{*}}{dt^{T}}+\frac{\partial M_{t}}{\partial t^{T}}\\
   & = & \left(\left.\frac{\partial M}{\partial m^{T}}\right|_{_{m=m_{t}^{*}}}+
   \frac{\partial}{\partial m^{T}}\left(\begin{array}{cc}
  I & \nabla f\end{array}\right)\left(\begin{array}{c}
  t_{0}\\
  t_{f}
  \end{array}\right)\right)\frac{dm^{*}}{dt^{T}}+\left(\begin{array}{cc}
  I & \nabla f\end{array}\right)
  \end{eqnarray*}
  The term $\frac{\partial}{\partial m^{T}}\left(\begin{array}{cc}
  I & \nabla f\end{array}\right)\left(\begin{array}{c}
  t_{0}\\
  t_{f}
  \end{array}\right)$ is awkward, but it disappears when we evaluate at $t=0$,
  giving
  \begin{eqnarray*}
  \frac{dm_{t}^{*}}{dt^{T}} & = &
    \left(\left.\frac{\partial M}{\partial m^{T}}\right|_{_{m=m_{t}^{*}}}\right)
    \frac{dm^{*}}{dt^{T}}+\left(\begin{array}{cc}
  I & \nabla f\end{array}\right)\\
   & = & \left(\frac{\partial^{2}E}{\partial m\partial m^{T}}+
    I\right)\frac{dm^{*}}{dt^{T}}+\left(\begin{array}{cc}
  I & \nabla f\end{array}\right) \Rightarrow \\
  \frac{dm^{*}}{dt^{T}} & = &
    -\left(\frac{\partial^{2}E}{\partial m\partial m^{T}}\right)^{-1}
    \left(\begin{array}{cc}I & \nabla f\end{array}\right)
  \end{eqnarray*}
  Recalling that
  \begin{eqnarray*}
  \frac{dm^{*}}{dt_{0}^{T}} & := & \lrcov
  \end{eqnarray*}
  We can plug in to see that
  \begin{equation}
  \lrcov_{\theta\phi} = \frac{dm^{*}}{dt_{f}} = \lrcov \nabla f
  \end{equation}
\end{proof}
Finally, suppose we are interested in estimating $\cov_p(\gamma(\theta),
\phi(\theta))$, where $g(m) := \mbeq\left[\gamma(\theta)\right]$.  Again using
the same reasoning that led to \eq{lrvb_derivative_defn}, we will define
\begin{equation*}
\truecov_{\gamma\phi} = \cov_p(\gamma(\theta), \phi(\theta))
\approx \frac{d \mbeq\left[\gamma(\theta)\right]}{dt_f} =: \lrcov_{\gamma\phi}
\end{equation*}
\begin{proposition} \label{prop:function_function_covariance}
  If $\mbeq\left[\phi(\theta)\right] = f(m)$ and
  $\mbeq\left[\gamma(\theta)\right] = g(m)$ are differentiable functions of $m$
  with gradients $\nabla f$ and $\nabla g$ respectively, then
  $$
    \lrcov_{\gamma\phi} = \nabla g^{T}\lrcov \nabla f
  $$
\end{proposition}
\begin{proof}
  By \lem{theta_function_covariance} an application of the chain rule,
  \begin{eqnarray*}
  \lrcov_{\gamma\phi} =
  \frac{d \mbeq\left[\gamma(\theta)\right]}{dt_f} =
  \frac{d g\left(m\right)}{dt_f} & = & \frac{dg(m)}{dm^{T}}\frac{dm}{dt_{f}}
    =  \nabla g^{T}\lrcov \nabla f
  \end{eqnarray*}
\end{proof}


\section{Exactness of LRVB for multivariate normal means} \label{app:mvn_exact}

For any target distribution $p(\theta | x)$, it is well-known that MFVB cannot
be used to estimate the covariances between the components of $\theta$. In
particular, if $q^*$ is the estimate of $p(\theta | x)$ returned by MFVB, $q^*$
will have a block-diagonal covariance matrix---no matter the form of the
covariance of $p(\theta | x)$.

Consider approximating a multivariate Gaussian posterior distribution
$p(\theta|x)$ with MFVB. The Gaussian is the unique distribution that is fully
determined by its mean and covariance. This posterior arises, for instance,
given a multivariate normal likelihood $p(x | \mu) = \prod_{n=1:N} \gauss(x_n |
\mu, S)$ with fixed covariance $S$ and an improper uniform prior on the mean
parameter $\mu$. We make the mean field factorization assumption
$q(\mu)=\prod_{d=1:D} q(\mu_d)$, where $D$ is the total dimension of $\mu$. This
fact is often used to illustrate the shortcomings of MFVB
\citep{wang:2005:inadequacy,bishop:2006:pattern,turner:2011:two}.
In this case, it is well known that the MFVB posterior means are correct, but the
marginal variances are underestimated if $S$ is not diagonal.
However, since the posterior means are correctly estimated,
the LRVB approximation in \eq{spec_lrvb} is in fact an equality.
That is, for this model,
$\lrcov = d \mpq_t / d t^T = \truecov$ exactly.

In order to prove this result, we will rely on the following lemma.
\begin{lemma} \label{lem:lrvb_mvn}
  Consider a target posterior distribution characterized by $p(\theta | x) =
  \gauss(\theta | \mu, \Sigma)$, where $\mu$ and $\Sigma$ may depend on $x$, and
  $\Sigma$ is invertible. Let $\theta = (\theta_{1}, \ldots, \theta_{J})$, and
  consider a MFVB approximation to $p(\theta| x)$ that factorizes as $q(\theta) =
  \prod_{j} q(\theta_j)$. Then the variational posterior means are the true
  posterior means; i.e. $m_j = \mu_j$ for all $j$ between $1$ and $J$.
\end{lemma}

\begin{proof}
  The derivation of MFVB for the multivariate normal can be found in Section
  10.1.2 of \citep{bishop:2006:pattern}; we highlight some key results here. Let
  $\Lambda = \Sigma^{-1}$. Let the $j$ index on a row or column correspond to
  $\theta_j$, and let the $-j$ index correspond to $\{\theta_{i}: i \in
  [J]\setminus j\}$. E.g., for $j=1$,
  $$
    \Lambda
      = \left[ \begin{array}{ll}
          \Lambda_{11} & \Lambda_{1,-1} \\
          \Lambda_{-1,1} & \Lambda_{-1,-1}
        \end{array} \right].
  $$
  By the assumption that $p(\theta | x) = \gauss(\theta | \mu, \Sigma)$, we have
\begin{eqnarray}\label{eq:mvn_variational_dist}
  \lefteqn{\log p(\theta_{j} | \theta_{i \in [J]\setminus j}, x)} \nonumber\\
      &=& -\frac{1}{2} (\theta_{j} - \mu_{j})^{T} \Lambda_{jj} (\theta_j - \mu_j) +
         (\theta_{j} - \mu_{j})^{T} \Lambda_{j,-j} (\theta_{-j} - \mu_{-j}) + \constant,
\end{eqnarray}
  where the final term is constant with respect to $\theta_{j}$.
  It follows that
  \begin{align*}
    \log q^{*}_{j}(\theta_j)
      &= \mbe_{q^{*}_{i}: i \in [J]\setminus j} \log p(\theta, x) + \constant \\
      &= -\frac{1}{2} \theta_{j}^{T} \Lambda_{jj} \theta_j + \theta_j \mu_j \Lambda_{jj} - \theta_j \Lambda_{j,-j} (\mbe_{q^{*}} \theta_{-j} - \mu_{-j}).
  \end{align*}
  So
  \begin{equation*}
    q^*_j(\theta_j) = \gauss(\theta_j | m_{j}, \Lambda_{jj}^{-1}),
  \end{equation*}
  with mean parameters
  \begin{equation} \label{eq:mvn_stable_point}
    m_{j} = \mbe_{q^{*}_j} \theta_j = \mu_{j} - \Lambda_{jj}^{-1} \Lambda_{j,-j} (m_{-j} - \mu_{-j})
  \end{equation}
  as well as an equation for $\mbe_{q^{*}} \theta^T \theta$.

Note that $\Lambda_{jj}$ must be invertible, for if it were not, $\Sigma$ would
not be invertible.

The solution $m = \mu$ is a unique stable point for \eq{mvn_stable_point}, since
the fixed point equations for each $j$ can be stacked and rearranged to give
\begin{eqnarray*}
m-\mu & = & -\left[\begin{array}{ccccc}
0 & \Lambda_{11}^{-1}\Lambda_{12} & \cdots & \Lambda_{11}^{-1}\Lambda_{1\left(J-1\right)} & \Lambda_{11}^{-1}\Lambda_{1J}\\
\vdots &  & \ddots &  & \vdots\\
\Lambda_{JJ}^{-1}\Lambda_{J1} & \Lambda_{JJ}^{-1}\Lambda_{J2} & \cdots & \Lambda_{JJ}^{-1}\Lambda_{J\left(J-1\right)} & 0
\end{array}\right]\left(m-\mu\right)\\
 & = & -\left[\begin{array}{ccccc}
\Lambda_{11}^{-1} & \cdots & 0 & \cdots & 0\\
\vdots & \ddots &  &  & \vdots\\
0 &  & \ddots &  & 0\\
\vdots &  &  & \ddots & \vdots\\
0 & \cdots & 0 & \cdots & \Lambda_{JJ}^{-1}
\end{array}\right]\left[\begin{array}{ccccc}
0 & \Lambda_{12} & \cdots & \Lambda_{1\left(J-1\right)} & \Lambda_{1J}\\
\vdots &  & \ddots &  & \vdots\\
\Lambda_{J1} & \Lambda_{J2} & \cdots & \Lambda_{J\left(J-1\right)} & 0
\end{array}\right]\left(m-\mu\right)\Leftrightarrow\\
0 & = & \left[\begin{array}{ccccc}
\Lambda_{11} & \cdots & 0 & \cdots & 0\\
\vdots & \ddots &  &  & \vdots\\
0 &  & \ddots &  & 0\\
\vdots &  &  & \ddots & \vdots\\
0 & \cdots & 0 & \cdots & \Lambda_{JJ}
\end{array}\right]\left(m-\mu\right) +\\
&& \left[\begin{array}{ccccc}
0 & \Lambda_{12} & \cdots & \Lambda_{1\left(J-1\right)} & \Lambda_{1J}\\
\vdots &  & \ddots &  & \vdots\\
\Lambda_{J1} & \Lambda_{J2} & \cdots & \Lambda_{J\left(J-1\right)} & 0
\end{array}\right]\left(m-\mu\right)\Leftrightarrow\\
0 & = & \Lambda \left(m-\mu\right) \Leftrightarrow\\
m & = & \mu.
\end{eqnarray*}
The last step follows from the assumption that $\Sigma$ (and hence $\Lambda$)
is invertible.  It follows that $\mu$ is the unique stable point of
\eq{mvn_stable_point}.

\end{proof}

\begin{proposition} \label{prop:lrvb_mvn}
  Assume we are in the setting of \lem{lrvb_mvn}, where additionally
  $\mu$ and $\Sigma$ are on the interior of the feasible parameter space.
  Then the LRVB covariance estimate exactly captures the true covariance,
  $\hat{\Sigma} = \Sigma$.

\end{proposition}

\begin{proof}

  Consider the perturbation for LRVB defined in \eq{perturbed_dens}.
  By perturbing the log likelihood, we change both the true means $\mu_t$
  and the variational solutions, $m_t$. The result is a valid
  density function since the original $\mu$ and $\Sigma$ are on the
  interior of the parameter space.
  By \lem{lrvb_mvn}, the MFVB solutions are exactly the true
  means, so $m_{t,j} = \mu_{t,j}$, and the derivatives are the same
  as well.  This means that the first term in \eq{spec_lrvb} is
  not approximate, i.e.
  \begin{equation*}
  \frac{d \mpq_{t}}{d t^{T}}
    = \frac{d}{d t^{T}} \mbe_{p_{t}} \theta
    = \truecov_{t},
  \end{equation*}
  It follows from the arguments above that the LRVB covariance
  matrix is exact, and $\hat{\Sigma} = \Sigma$.

\end{proof}

\section{Comparison with supplemented expectation-maximization}\label{app:SEM}

The result in \app{mvn_exact} about the multivariate normal distribution
draws a connection between LRVB
corrections and the ``supplemented expectation-maximization'' (SEM)
method of \citep{meng:1991:using}.  SEM is an asymptotically
exact covariance correction for the EM algorithm that transforms
the full-data Fisher information matrix into the observed-data Fisher
information matrix using a correction that is formally similar to
\eq{spec_lrvb}.  In this section, we argue that this similarity
is not a coincidence; in fact the SEM correction is an
asymptotic version of LRVB with two variational blocks,
one for the missing data and one for the unknown parameters.

Although LRVB as described here requires a prior
(unlike SEM, which supplements the MLE),
the two covariance corrections coincide when
the full information likelihood is approximately log quadratic
and proportional to the posterior, $p(\theta \vert x)$.
This might be expected to occur when we have a large number
of independent data points informing each parameter---i.e.,
when a central limit theorem applies and the priors do not
affect the posterior.
In the full information likelihood, some
terms may be viewed as missing data, whereas in the Bayesian
model the same terms may be viewed as latent parameters,
but this does not prevent us from formally comparing the two methods.

We can draw a term-by-term analogy with
the equations in \citep{meng:1991:using}. We denote variables
from the SEM paper with a superscript ``$SEM$'' to avoid confusion.
MFVB does not differentiate between missing
data and parameters to be estimated, so our $\theta$ corresponds to
$(\theta^{SEM}, Y_{mis}^{SEM})$ in \citep{meng:1991:using}.
SEM is an asymptotic
theory, so we may assume that $(\theta^{SEM}, Y_{mis}^{SEM})$ have a
multivariate normal
distribution, and that we are interested in the mean and covariance of
$\theta^{SEM}$.

In the E-step of \citep{meng:1991:using}, we replace $Y_{mis}^{SEM}$ with
its conditional expectation given the data and other $\theta^{SEM}$.
This corresponds precisely to \eq{mvn_stable_point}, taking
$\theta_j = Y_{mis}^{SEM}$.  In the M-step, we find the maximum
of the log likelihood with respect to $\theta^{SEM}$, keeping
$Y_{mis}^{SEM}$ fixed at its expectation.  Since the mode
of a multivariate normal distribution is also its mean,
this, too, corresponds to \eq{mvn_stable_point}, now taking
$\theta_j = \theta^{SEM}$.

It follows that the MFVB and EM fixed point equations are the same;
i.e., our $M$ is the same as their $M^{SEM}$, and
our $\partial M / \partial m$ of \eq{dM_dt} corresponds
to the transpose of their $DM^{SEM}$, defined in \eqw{2.2.1}
of \citep{meng:1991:using}.  Since the ``complete information'' corresponds to
the variance of $\theta^{SEM}$ with fixed values for $Y_{OBS}^{SEM}$,
this is the same as our $\Sigma_{q^*,11}$, the variational covariance,
whose inverse is $I_{oc}^{-1}$.  Taken all together, this means that
equation (2.4.6) of \citep{meng:1991:using} can be
re-written as our \eq{spec_lrvb}.
\begin{align*}
V^{SEM} =& I_{oc}^{-1} \left(I - DM^{SEM}\right)^{-1} \Rightarrow\\
\Sigma =& \vbcov \left(I - \left(\frac{\partial M}{\partial m^T}\right)^T \right)^{-1}
       = \left(I - \frac{\partial M}{\partial m^T} \right)^{-1} \vbcov
\end{align*}

\section{Normal-Poisson details} \label{app:np_details}

In this section, we use this model to provide a detailed, step-by-step description of
a simple LRVB analysis.

The full joint distribution for the model in \eq{pn_model} is
\begin{align*}
\log p\left(y,z,\beta,\tau\right) &= \sum_{n=1}^{N}\left(-\frac{1}{2}\tau z_{n}^{2}+x_{n}\tau\beta z_{n}-\frac{1}{2}x_{n}^{2}\tau\beta^{2}-\frac{1}{2}\log\tau\right)\\
 &+\sum_{n=1}^{N}\left(-\exp\left(z_{n}\right)+z_{n}y_{n}\right)
 -\frac{1}{2\sigma_{\beta}^{2}}\beta^{2}+\left(\alpha_{\tau}-1\right)\log\tau-\beta_{\tau}\tau+\constant
\end{align*}
We find a mean-field approximation under the factorization
$q\left(\beta,\tau,z\right) =
q\left(\beta\right)q\left(\tau\right)\prod_{n=1}^{N}q\left(z_{n}\right)$. By
inspection, the log joint is quadratic in $\beta$, so the optimal
$q\left(\beta\right)$ will be Gaussian \citep{bishop:2006:pattern}. Similarly, the
log joint is a function of $\tau$ only via $\tau$ and $\log\tau$, so the optimal
$q\left(\tau\right)$ will be gamma. However, the joint does not take a standard
exponential family form in $z_n$:
$$
  \log p\left(z_{n}\vert y,\beta,\tau\right) = \left(x_{n}\tau\beta+y_{n}\right)z_{n}-\frac{1}{2}\tau z_{n}^{2}-\exp\left(z_{n}\right)+\constant
$$
The difficulty is with the term $\exp\left(z_{n}\right)$. So we make the further
restriction that
$$
  q\left(z_{n}\right) = \gauss\left(\cdot\right)=q\left(z_{n};\mbe\left[z_{n}\right],\mbe\left[z_{n}^{2}\right]\right).
$$
Fortunately, the troublesome term has an analytic expectation, as
a function of the mean parameters, under this variational posterior:
$$
  \mbeq\left[\exp\left(z_{n}\right)\right] = \exp\left(\mbeq\left[z_{n}\right]+\frac{1}{2}\left(\mbeq\left[z_{n}^{2}\right]-\mbeq\left[z_{n}\right]^{2}\right)\right).
$$
We can now write the variational distribution in terms of the following
mean parameters:
$$
  m = \left(\mbeq\left[\beta\right],\mbeq\left[\beta^{2}\right],\mbeq\left[\tau\right],\mbeq\left[\log\tau\right],\mbeq\left[z_{1}\right],\mbeq\left[z_{1}^{2}\right],...,\mbeq\left[z_{N}\right],\mbeq\left[z_{N}^{2}\right]\right)^{T}.
$$
Calculating the LRVB covariance consists of roughly four steps:

\begin{enumerate}
\item finding the MFVB optimum $q^{*}$,
\item computing the covariance $\vbcov$ of $q^*$,
\item computing $H$, the Hessian of $L(m)$, for $q^*$, and
\item computing the matrix inverse and solving $\left(I-VH\right)^{-1}V$.
\end{enumerate}

For step (1), the LRVB correction is agnostic as to how the optimum
is found. In our experiments below, we
follow a standard
coordinate ascent procedure for MFVB \citep{bishop:2006:pattern}. We analytically update
$q\left(\beta\right)$ and $q\left(\tau\right)$.
Given $q\left(\beta\right)$ and $q\left(\tau\right)$, finding the
optimal $q\left(z\right)$ becomes $N$ separate two-dimensional optimization
problems; there is one dimension for each of the mean parameters $\mbeq \left[ z_n \right]$ and $\mbeq \left[ z_n^2 \right]$.
In our examples, we solved these problems sequentially using
IPOPT \citep{ipopt:package}.

To compute $\vbcov$ for step (2),
we note that by the mean-field assumption,
$\beta$, $\tau$, and $z_{n}$ are independent, so $\vbcov$ is block
diagonal. Since we have chosen convenient variational distributions,
the mean parameters have known covariance matrices. For example, from
standard properties of the normal distribution,
$\textrm{Cov}\left(\beta,\beta^{2}\right)=2\mbeq\left[\beta\right]$$\left(\mbeq\left[\beta^{2}\right]-\mbeq\left[\beta\right]^{2}\right)$.

For step (3), the mean parameters for $\beta$ and $\tau$ co-occur with each other
and with all the $z_{n}$, so these four rows of $H$ are expected
to be dense. However, the mean parameters for $z_{n}$ never occur
with each other, so the bulk of $H$---the $2N\times2N$ block corresponding
to the mean parameters of $z$---will be block diagonal (\fig{H_sparse}).
The Hessian of $L\left(m\right)$ can be calculated analytically,
but we used the autodifferentiation software \texttt{JuMP} \citep{JuMP:LubinDunningIJOC}.

Finally, for step (4),
we use the technique in \mysec{scaling_formulas}
to exploit the sparsity of $\vbcov$ and $H$ (\fig{IVH_sparse}) in calculating $(I-VH)^{-1}$.

\begin{figure}[ht!]
  \centering
  \begin{subfigure}{0.3\linewidth}
    \centering
    \includegraphics[height=0.3 \linewidth]{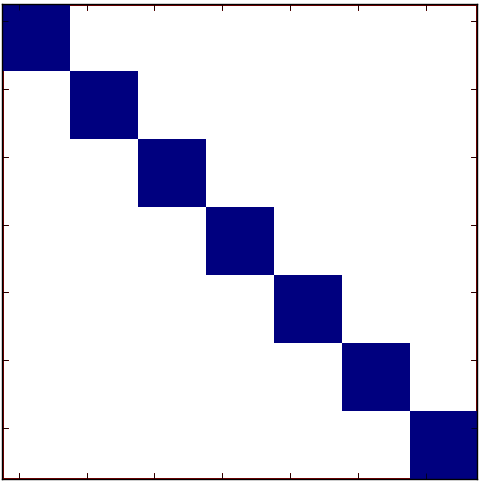}
    \caption{\label{fig:V_sparse} MFVB covariance $V$}
  \end{subfigure}
  \begin{subfigure}{0.3\linewidth}
    \centering
    \includegraphics[height=0.3 \linewidth]{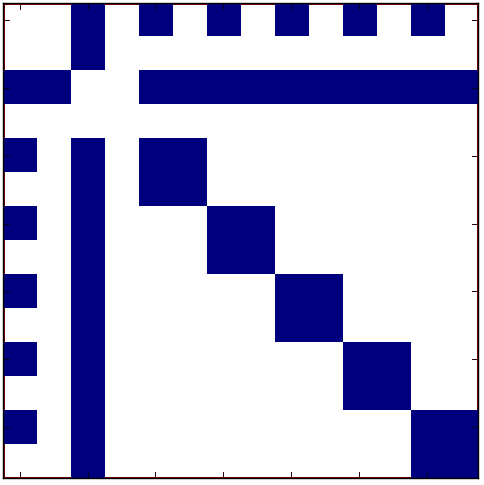}
    \caption{\label{fig:H_sparse} Hessian matrix $H$}
  \end{subfigure}
  \begin{subfigure}{0.3\linewidth}
    \centering
    \includegraphics[height=0.3 \linewidth]{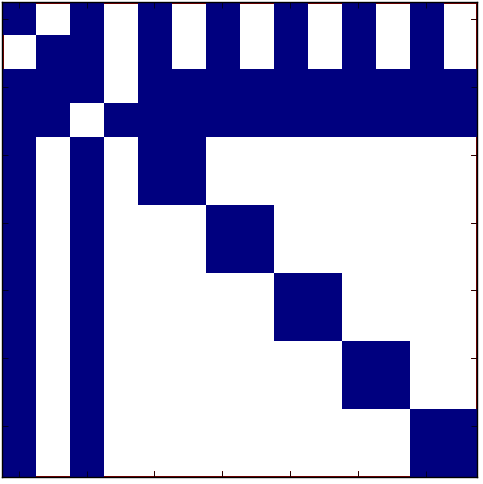}
    \caption{\label{fig:IVH_sparse} \mbox{$(I - VH)$}}
  \end{subfigure}
  \caption{Sparsity patterns for $\hat\truecov = (I - VH)^{-1}$ using the model in \eq{pn_model}, $n = 5$ (white = 0)}
  \label{fig:sparsity_patterns}
\end{figure}

\section{Random effects model details} \label{app:re_details}

As introduced in \mysec{random_effects_model}, our model is:
\begin{eqnarray*}
y_{n}\vert\beta,z,\tau & \indep & \gauss\left(\beta^{T}x_{n}+r_{n}z_{k\left(n\right)},\tau^{-1}\right)\\
z_{k}\vert\nu & \iid & \gauss\left(0,\nu^{-1}\right)
\end{eqnarray*}
With the priors:
\begin{eqnarray*}
\beta & \sim & \gauss\left(0,\Sigma_{\beta}\right)\\
\nu & \sim & \Gamma\left(\alpha_{\nu},\beta_{\nu}\right)\\
\tau & \sim & \Gamma\left(\alpha_{\tau},\beta_{\tau}\right)
\end{eqnarray*}
We will make the following mean field assumption:
\begin{eqnarray*}
q\left(\beta,z,\tau,\nu\right) & = & q\left(\nu\right)q\left(\tau\right)q\left(\beta\right)\prod_{k=1}^{K}q\left(z_{k}\right)
\end{eqnarray*}
We have $n\in\left\{ 1,...,N\right\} $, and $k\in\left\{ 1,...,K\right\} $,
and $k\left(n\right)$ matches an observation $n$ to a random effect
$k$, allowing repeated observations of a random effect. The full
joint log likelihood is:
\begin{eqnarray*}
\log p\left(y_{n}\vert z_{k\left(n\right)},\tau,\beta\right) & = & -\frac{\tau}{2}\left(y_{n}-\beta^{T}x_{n}-r_{n}z_{k\left(n\right)}\right)^{2}+\frac{1}{2}\log\tau+\constant\\
\log p\left(z_{k}\vert\nu\right) & = & -\frac{\nu}{2}z_{k}^{2}+\frac{1}{2}\log\nu+\constant\\
\log p\left(\beta\right) &  & -\frac{1}{2}\textrm{trace}\left(\Sigma_{\beta}^{-1}\beta\beta^{T}\right)+\constant\\
\log p\left(\tau\right) & = & \left(\alpha_{\tau}-1\right)\log\tau-\beta_{\tau}\tau+\constant\\
\log p\left(\nu\right) & = & \left(\alpha_{\nu}-1\right)\log\nu-\beta_{\nu}\nu+\constant\\
\log p\left(y,\tau,\beta,z\right) & = & \sum_{n=1}^{N}\log p\left(y_{n}\vert z_{k\left(n\right)},\tau,\beta\right)+\sum_{k=1}^{K}\log p\left(z_{k}\vert\nu\right)+\\
 &  & \log p\left(\beta\right)+\log p\left(\nu\right)+\log p\left(\tau\right)
\end{eqnarray*}
Expanding the first term of the conditional likelihood of $y_{n}$
gives
\begin{eqnarray*}
\lefteqn{-\frac{\tau}{2}\left(y_{n}-\beta^{T}x_{n}-r_{n}z_{k\left(n\right)}\right)^{2}}\\
& = & -\frac{\tau}{2}\left(y_{n}^{2}-2y_{n}x_{n}^{T}\beta-2y_{n}r_{n}z_{n\left(k\right)}+\textrm{trace}\left(x_{n}x_{n}^{T}\beta\beta^{T}\right)+r_{n}^{2}z_{k\left(n\right)}^{2}+2r_{n}x_{n}^{T}\beta z_{k\left(n\right)}\right)
\end{eqnarray*}
By grouping terms, we can see that the mean parameters will be
\begin{eqnarray*}
q\left(\beta\right) & = & q\left(\beta;\mbeq\left[\beta\right],\mbeq\left[\beta\beta^{T}\right]\right)\\
q\left(z_{k}\right) & = & q\left(z_{k};\mbeq\left[z_{k}\right],\mbeq\left[z_{k}^{2}\right]\right)\\
q\left(\tau\right) & = & q\left(\tau;\mbeq\left[\tau\right],\mbeq\left[\log\tau\right]\right)\\
q\left(\nu\right) & = & q\left(\nu;\mbeq\left[\nu\right],\mbeq\left[\log\nu\right]\right)
\end{eqnarray*}
It follows that the optimal variational distributions are $q\left(\beta\right)=$multivariate
normal, $q\left(z_{k}\right)=$univariate normal, and $q\left(\tau\right)$
and $q\left(\nu\right)$ will be gamma. We performed standard coordinate
ascent on these distributions \citep{bishop:2006:pattern}.

As in \mysec{normal_poisson_model}, we implemented this model in the
autodifferentiation software JuMP \citep{JuMP:LubinDunningIJOC}.
This means conjugate coordinate updates were easy, since the natural
parameters corresponding to a mean parameters are the first derivatives
of the log likelihood with respect to the mean parameters. For example,
denoting the log likelihood at step $s$ by $L_{s}$, the update for
$q_{s+1}\left(z_{k}\right)$ will be:
\begin{eqnarray*}
\log q_{s+1}\left(z_{k}\right) & = & \frac{\partial\mbeq\left[L_{s}\right]}{\partial\mbeq\left[z_{k}\right]}z_{k}+\frac{\partial\mbeq\left[L_{s}\right]}{\partial\mbeq\left[z_{k}^{2}\right]}z_{k}^{2}+\constant
\end{eqnarray*}
Given the partial derivatives of $L_{s}$ with respect to the mean
parameters, the updated mean parameters for $z_{k}$ can be read off
directly using standard properties of the normal distribution.

The variational covariance matrices are all standard. We can see that
$H$ will have nonzero terms in general (for example, the three-way
interaction $\mbeq\left[\tau\right]\mbeq\left[z_{k\left(n\right)}\right]\mbeq\left[\beta\right]$),
and that LRVB will be different from MFVB. As usual in our models,
$H$ is sparse, and we can easily apply the technique in section \mysec{scaling_formulas}
to get the covariance matrix excluding the random effects, $z$.

\section{Multivariate normal mixture details} \label{app:mvn_details}

In this section we derive the basic formulas needed to calculate \eq{spec_lrvb}
for a finite mixture of normals, which
is the model used in \mysec{experiments}.  We will
follow the notation introduced in \mysec{normal_mixture_model}.

Let each observation, $x_{n}$, be a $P\times1$ vector. We will denote
the $P$th component of the $n$th observation $x_{n}$, with a
similar pattern for $z$ and $\mu$. We will denote the $p$, $q$th
entry in the matrix $\Lambda_{k}$ as $\Lambda_{k,pq}$. The data
generating process is as follows:
\begin{eqnarray*}
P\left(x | \mu, \pi, \Lambda \right) &=&
  \prod_{n=1}^N P\left(x_{n}|z_{n},\mu,\Lambda \right)
  \prod_{k=1}^{K} P\left(z_{nk}|\pi_{k} \right)\\
\log P\left(x_{n}|z_{n},\mu,\Lambda\right) & = &
    \sum_{n=1}^{N}z_{nk}\log\phi_{k}(x_{n}) + \constant\\
\log\phi_{k}(x) & = & -\frac{1}{2}\left(x - \mu_{k}\right)^{T} \Lambda_{k}\left(x-\mu_{k}\right) +
    \frac{1}{2}\log\left|\Lambda_{k}\right|+ \constant\\
\log P(z_{nk}|\pi_{k}) & = & \sum_{k=1}^{K}z_{nk}\log\pi_{k} + \constant
\end{eqnarray*}
It follows that the log posterior is given by
\begin{eqnarray*}
\log P(z,\mu,\pi,\Lambda | x) & = & \sum_{n=1}^{N}\sum_{k=1}^{K}z_{nk}\left(\log\pi_{k}-\frac{1}{2}\left(x_{n}-\mu_{k}\right)^{T}\Lambda_{k}\left(x_{n}-\mu_{k}\right)+\frac{1}{2}\log\left|\Lambda_{k}\right|\right) + \\
  &  & \sum_{k=1}^{K} \log p(\mu_{k}) + \sum_{k=1}^{K} \log p(\Lambda_{k}) +
        \log p(\pi) + \constant
\end{eqnarray*}
We used a multivariate normal prior for $\mu_{k}$, a Wishart prior for
$\Lambda_{k}$, and a Dirichlet prior for $\pi$.  In the simulations described
in \mysec{normal_mixture_model}, we used the following prior parameters
for the VB model:
\begin{eqnarray*}
  p(\mu_{k}) &=& \mathcal{N}\left(0_P, \textrm{diag}_P(0.01)^{-1}\right) \\
  p(\Lambda_{k}) &=& \textrm{Wishart}(\textrm{diag}_P(0.01), 1)\\
  p(\pi) &=& \textrm{Dirichlet}(5_K)
\end{eqnarray*}
Here, $\textrm{diag}_P(a)$ is a $P$-dimensional diagonal matrix with $a$ on the
diagonal, and $0_P$ is a length $P$ vector of the value $0$, with a similar
definition for $5_K$. Unfortunately, the function we used for the MCMC
calculations, \texttt{rnmixGibbs} in the package \texttt{bayesm}, uses a
different form for the $\mu_{k}$ prior. Specifically, \texttt{rnmixGibbs} uses
the prior
$$
  p_{MCMC}\left(\mu_{k} \right \vert \Lambda_{k}) =
    \mathcal{N}(0, a^{-1} \Lambda_{k}^{-1})
$$
where $a$ is a scalar.  There is no way to exactly match
$p_{MCMC}(\mu_k)$ to $p(\mu_k)$, so we simply set $a=0.01$.
Since our datasets are all reasonably large, the prior was dominated by the
likelihood, and we found the results extremely insensitive to the prior
on $\mu_{k}$, so this discrepancy is of no practical importance.

The parameters $\mu_{k}$, $\Lambda_{k}$, $\pi$, and $z_{n}$ will
each be given their own variational distribution.  For $q_{\mu_k}$ we will
use a multivariate normal distribution; for $q_{\Lambda_{k}}$ we will
us a Wishart distirbution; for $q_{\pi}$ we will use a Dirichlet distribution;
for $q_{z_{n}}$ we will use a Multinoulli (a single multinomial draw).  These
are all the optimal variational choices given the mean field assumption and
the conditional conjugacy in the model.

The sufficient statistics for $\mu_{k}$ are all terms of the form
$\mu_{kp}$ and $\mu_{kp}\mu_{kq}$. Consequently, the sub-vector
of $\theta$ corresponding to $\mu_{k}$ is
\begin{eqnarray*}
\theta_{\mu_{k}} & = & \left(\begin{array}{c}
\mu_{k1}\\
\vdots\\
\mu_{kp}\\
\mu_{k1}\mu_{k1}\\
\mu_{k1}\mu_{k2}\\
\vdots\\
\mu_{kP}\mu_{kP}
\end{array}\right)
\end{eqnarray*}
We will only save one copy of $\mu_{kp}\mu_{kq}$ and $\mu_{kq}\mu_{kp}$,
so $\theta_{\mu_{k}}$ has length $P+\frac{1}{2}\left(P+1\right)P$.
For all the parameters, we denote the complete stacked vector without
a $k$ subscript:
\begin{eqnarray*}
\theta_{\mu} & = & \left(\begin{array}{c}
\theta_{\mu_{1}}\\
\vdots\\
\theta_{\mu_{K}}
\end{array}\right)
\end{eqnarray*}
The sufficient statistics for $\Lambda_{k}$ are all the terms $\Lambda_{k,pq}$
and the term $\log\left|\Lambda_{k}\right|$. Again, since $\Lambda$ is
symmetric, we do not keep redundant terms, so $\theta_{\Lambda_{k}}$ has length
$1+\frac{1}{2}\left(P+1\right)P$. The sufficient statistic for $\pi$ is the
$K$-vector $\left(\log\pi_{1},...,\log\pi_{K}\right)$. The sufficient statistics
for $z$ are simply the $N\times K$ values $z_{nk}$ themselves.

In terms of \mysec{scaling_formulas}, we have
\begin{eqnarray*}
\alpha & = & \left(\begin{array}{c}
\theta_{\mu}\\
\theta_{\Lambda}\\
\theta_{\pi}
\end{array}\right)\\
z & = & \left(\begin{array}{c}
\theta_{z}\end{array}\right)
\end{eqnarray*}
That is, we are primarily interested in the covariance of the sufficient
statistics of $\mu$, $\Lambda$, and $\pi$.  The latent variables $z$ are
nuisance parameters.

To put the log likelihood in terms useful for LRVB, we must express
it in terms of the sufficient statistics, taking into account the
fact the $\theta$ vector does not store redundant terms (e.g. it
will only keep $\Lambda_{ab}$ for $a<b$ since $\Lambda$ is symmetric).
\begin{eqnarray*}
  \lefteqn{-\frac{1}{2}\left(x_{n}-\mu_{k}\right)^{T}\Lambda_{k}\left(x_{n}-\mu_{k}\right)} \\
 & = & -\frac{1}{2}\textrm{trace}\left(\Lambda_{k}\left(x_{n}-\mu_{k}\right)\left(x_{n}-\mu_{k}\right)^{T}\right)\\
 & = & -\frac{1}{2}\sum_{a}\sum_{b}\left(\Lambda_{k,ab}\left(x_{n,a}-\mu_{k,a}\right)\left(x_{n,b}-\mu_{k,b}\right)\right)\\
 & = & -\frac{1}{2}\sum_{a}\sum_{b}\left(\Lambda_{k,ab}\mu_{k,a}\mu_{k,b}-\Lambda_{k,ab}x_{n,a}\mu_{k,b}-
          \Lambda_{k,ab}x_{n,b}\mu_{k,a}+\Lambda_{k,ab}x_{n,a}x_{n,b}\right)\\
 & = & -\frac{1}{2}\sum_{a}\Lambda_{k,aa}\left(\mu_{k}^{2}\right)^{a}+\sum_{a}\Lambda_{k,aa}x_{n,a}\mu_{k,a}-\frac{1}{2}\sum_{a}\Lambda_{k,aa}\left(x_{n}^{2}\right)^{2}-\\
 &  & \frac{1}{2}\sum_{a\ne b}\Lambda_{k,ab}\mu_{k,a}\mu_{k,b}+\sum_{a\ne b}\Lambda_{k,ab}x_{n,a}\mu_{k,b}-\frac{1}{2}\sum_{a\ne b}\Lambda_{k,ab}x_{n,a}x_{n,b}\\
 & = & -\frac{1}{2}\sum_{a}\Lambda_{k,aa}\left(\mu_{k}^{2}\right)^{a}+\sum_{a}\Lambda_{k,aa}x_{n,a}\mu_{k,a}-\frac{1}{2}\sum_{a}\Lambda_{k,aa}\left(x_{n}^{2}\right)^{2}-\\
 &  & \sum_{a<b}\Lambda_{k,ab}\mu_{k,a}\mu_{k,b}+\sum_{a<b}\Lambda_{k,ab}\left(x_{n,a}\mu_{k,b}+x_{n,b}\mu_{k,a}\right)-\sum_{a<b}\Lambda_{k,ab}x_{n,a}x_{n,b}
\end{eqnarray*}
The MFVB updates and covariances in $V$ are all given by properties of standard
distributions. To compute the LRVB corrections, it only remains to calculate the
Hessian, $H$. These terms can be read directly off the posterior. First we
calculate derivatives with respect to components of $\mu$.
\begin{eqnarray*}
\frac{\partial^{2}H}{\partial\mu_{k,a}\partial\Lambda_{k,ab}} & = & \sum_{i}z_{nk}x_{n,b}\\
\frac{\partial^{2}H}{\partial\left(\mu_{k,a}\mu_{k,b}\right)\partial\Lambda_{k,ab}} & = & -\left(\frac{1}{2}\right)^{1(a=b)}\sum_{n}z_{nk}\\
\frac{\partial^{2}H}{\partial\mu_{k,a}\partial z_{nk}} & = & \sum_{b}\Lambda_{k,ab}x_{n,b}\\
\frac{\partial^{2}H}{\partial\left(\mu_{k,a}\mu_{k,b}\right)\partial z_{nk}} & = & -\left(\frac{1}{2}\right)^{1(a=b)}\Lambda_{k,ab}
\end{eqnarray*}
All other $\mu$ derivatives are zero. For $\Lambda$,
\begin{eqnarray*}
\frac{\partial^{2}H}{\partial\Lambda_{k,ab}\partial z_{nk}} & = & -\left(\frac{1}{2}\right)^{1(a=b)}\left(x_{n,a}x_{n,b}-\mu_{k,a}x_{n,b}-\mu_{k,b}x_{n,a}+\mu_{k,a}\mu_{k,b}\right)\\
\frac{\partial^{2}H}{\partial\log\left|\Lambda_{k}\right|\partial z_{nk}} & = & \frac{1}{2}
\end{eqnarray*}
The remaining $\Lambda$ derivatives are zero. The only nonzero second
derivatives for $\log\pi$ are to $Z$ and are given by
\begin{eqnarray*}
\frac{\partial^{2}H}{\partial\log\pi_{k}\partial z_{nk}} & = & 1
\end{eqnarray*}
Note in particular that $H_{zz} = 0$, allowing efficient calculation of
\eq{nuisance_lrvb_est}.

\section{MNIST details} \label{app:mnist_details}

For a real-world example,
we applied LRVB to the unsupervised classification of two digits
from the MNIST dataset of handwritten digits.
We first preprocess the MNIST dataset by performing principle component
analysis on the training data's centered pixel intensities
and keeping the top $\MNISTp$ components.
For evaluation, the test data is projected onto the same
$\MNISTp$-dimensional subspace found using the training data.

We then treat the problem of
separating handwritten $0$s from $1$s as an unsupervised clustering
problem.  We limit the dataset to instances labeled as $0$
or $1$, resulting in $\MNISTn$ training and $\MNISTTestN$ test points.
We fit the training data
as a mixture of multivariate Gaussians.  Here, $K=2$, $P=\MNISTp$, and
$N=\MNISTn$.  Then, keeping the $\mu$, $\Lambda$, and $\pi$
parameters fixed, we calculate the expectations of the
latent variables $z$ in \eq{normal_mixture_model} for the test set.
We assign test set data point $x_n$ to whichever component has
maximum a posteriori expectation.  We count successful classifications
as test set points that match their cluster's majority label
and errors as test set points that are different from their cluster's
majority label.  By this measure, our test set error rate was
$\MNISTTestError$. We stress that we intend only to demonstrate
the feasibility of LRVB on a large, real-world dataset rather than
to propose practical methods for modeling MNIST.

\end{document}